%% file: bcbc.tex
\PassOptionsToPackage{unicode}{hyperref}
\PassOptionsToPackage{hyphens}{url}
\PassOptionsToPackage{dvipsnames,svgnames,x11names}{xcolor}
\documentclass[
  12pt]{article}

\usepackage{amsmath,amssymb}
\usepackage{iftex}

\ifPDFTeX
  \usepackage[T1]{fontenc}
  \usepackage[utf8]{inputenc}
  \usepackage{textcomp} 
\else 
  \usepackage{unicode-math}
  \defaultfontfeatures{Scale=MatchLowercase}
  \defaultfontfeatures[\rmfamily]{Ligatures=TeX,Scale=1}
\fi
\usepackage{lmodern}
\ifPDFTeX\else  
\fi
\IfFileExists{upquote.sty}{\usepackage{upquote}}{}
\IfFileExists{microtype.sty}{
  \usepackage[]{microtype}
  \UseMicrotypeSet[protrusion]{basicmath} 
}{}
\makeatletter
\@ifundefined{KOMAClassName}{
  \IfFileExists{parskip.sty}{%
    \usepackage{parskip}
  }{
    \setlength{\parindent}{0pt}
    \setlength{\parskip}{6pt plus 2pt minus 1pt}}
}{
  \KOMAoptions{parskip=half}}
\makeatother
\usepackage{xcolor}
\makeatletter
\ifx\paragraph\undefined\else
  \let\oldparagraph\paragraph
  \renewcommand{\paragraph}{
    \@ifstar
      \xxxParagraphStar
      \xxxParagraphNoStar
  }
  \newcommand{\xxxParagraphStar}[1]{\oldparagraph*{#1}\mbox{}}
  \newcommand{\xxxParagraphNoStar}[1]{\oldparagraph{#1}\mbox{}}
\fi
\ifx\subparagraph\undefined\else
  \let\oldsubparagraph\subparagraph
  \renewcommand{\subparagraph}{
    \@ifstar
      \xxxSubParagraphStar
      \xxxSubParagraphNoStar
  }
  \newcommand{\xxxSubParagraphStar}[1]{\oldsubparagraph*{#1}\mbox{}}
  \newcommand{\xxxSubParagraphNoStar}[1]{\oldsubparagraph{#1}\mbox{}}
\fi
\makeatother

\providecommand{\tightlist}{%
  \setlength{\itemsep}{0pt}\setlength{\parskip}{0pt}}\usepackage{longtable,booktabs,array}
\usepackage{calc} 
\usepackage{etoolbox}
\makeatletter
\patchcmd\longtable{\par}{\if@noskipsec\mbox{}\fi\par}{}{}
\makeatother
\IfFileExists{footnotehyper.sty}{\usepackage{footnotehyper}}{\usepackage{footnote}}
\makesavenoteenv{longtable}
\usepackage{graphicx}
\makeatletter
\def\maxwidth{\ifdim\Gin@nat@width>\linewidth\linewidth\else\Gin@nat@width\fi}
\def\maxheight{\ifdim\Gin@nat@height>\textheight\textheight\else\Gin@nat@height\fi}
\makeatother
\setkeys{Gin}{width=\maxwidth,height=\maxheight,keepaspectratio}
\makeatletter
\def\fps@figure{htbp}
\makeatother

\makeatletter
\@ifpackageloaded{caption}{}{\usepackage{caption}}
\AtBeginDocument{%
\ifdefined\contentsname
  \renewcommand*\contentsname{Table of contents}
\else
  \newcommand\contentsname{Table of contents}
\fi
\ifdefined\listfigurename
  \renewcommand*\listfigurename{List of Figures}
\else
  \newcommand\listfigurename{List of Figures}
\fi
\ifdefined\listtablename
  \renewcommand*\listtablename{List of Tables}
\else
  \newcommand\listtablename{List of Tables}
\fi
\ifdefined\figurename
  \renewcommand*\figurename{Figure}
\else
  \newcommand\figurename{Figure}
\fi
\ifdefined\tablename
  \renewcommand*\tablename{Table}
\else
  \newcommand\tablename{Table}
\fi
}
\@ifpackageloaded{float}{}{\usepackage{float}}
\floatstyle{ruled}
\@ifundefined{c@chapter}{\newfloat{codelisting}{h}{lop}}{\newfloat{codelisting}{h}{lop}[chapter]}
\floatname{codelisting}{Listing}

\makeatother
\makeatletter
\@ifpackageloaded{caption}{}{\usepackage{caption}}
\@ifpackageloaded{subcaption}{}{\usepackage{subcaption}}
\makeatother

\ifLuaTeX
  \usepackage{selnolig}  
\fi
\usepackage[]{natbib}
\usepackage{bookmark}

\IfFileExists{xurl.sty}{\usepackage{xurl}}{} 
\hypersetup{
  pdftitle={Title},
  pdfauthor={Author 1; Author 2},
  pdfkeywords={3 to 6 keywords, that do not appear in the title},
  colorlinks=true,
  linkcolor={blue},
  filecolor={Maroon},
  citecolor={Blue},
  urlcolor={Blue},
  pdfcreator={LaTeX via pandoc}}

\newcommand{\anon}{1}

\usepackage{delimset} 
\usepackage{bbm} 
\usepackage{algorithm} 
\usepackage{algpseudocode}
\usepackage{amsthm} 
\usepackage{enumitem} 

\DeclareMathOperator*{\argmin}{arg\,min}

\newtheorem{theorem}{Theorem}

\newtheorem{lemma}{Lemma}
\newtheorem{definition}{Definition}
\newtheorem{proposition}{Proposition}

\begin{document}

\def\spacingset#1{\renewcommand{\baselinestretch}%
{#1}\small\normalsize} \spacingset{1}


\if1\anon
{
  \title{\bf Biconvex Biclustering}
  \author{Sam Rosen\hspace{.2cm}\\
    Department of Statistical Science, Duke University\\
    Eric C. Chi \\
    School of Statistics, University of Minnesota\\
    Jason Xu \\
    Department of Biostatistics, University of California Los Angeles }
  \maketitle
} \fi

\if0\anon
{
  \bigskip
  \bigskip
  \bigskip
  \begin{center}
    {\LARGE\bf Biconvex Biclustering}
\end{center}
  \medskip
} \fi

\bigskip
\begin{abstract}
This article proposes a biconvex modification to convex biclustering in order to improve its performance in high-dimensional settings. In contrast to heuristics that discard a subset of noisy features \textit{a priori}, our method jointly learns and accordingly weighs informative features while discovering biclusters. Moreover, the method is adaptive to the data, and is accompanied by an efficient algorithm based on proximal alternating minimization, complete with detailed guidance on hyperparameter tuning and efficient solutions to optimization subproblems. These contributions are theoretically grounded; we establish finite-sample bounds on the objective function under sub-Gaussian errors, and generalize these guarantees to cases where input affinities need not be uniform. Extensive simulation results reveal our method consistently recovers underlying biclusters while weighing and selecting features appropriately, outperforming peer methods. An application to a gene microarray dataset of lymphoma samples recovers biclusters matching an underlying classification, while giving additional interpretation to the mRNA samples via the column groupings and fitted weights.

\end{abstract}

\noindent%
{\it Keywords:} clustering, optimization, feature weighing and selection, proximal algorithms, algebraic connectivity
\vfill

\newpage
\spacingset{1.1}

\newcommand{\ba}{\boldsymbol{a}}
\newcommand{\bb}{\boldsymbol{b}}
\newcommand{\bc}{\boldsymbol{c}}
\newcommand{\bd}{\boldsymbol{d}}
\newcommand{\be}{\boldsymbol{e}}
\newcommand{\bff}{\boldsymbol{f}}
\newcommand{\bg}{\boldsymbol{g}}
\newcommand{\bh}{\boldsymbol{h}}
\newcommand{\bi}{\boldsymbol{i}}
\newcommand{\bj}{\boldsymbol{j}}
\newcommand{\bk}{\boldsymbol{k}}
\newcommand{\bl}{\boldsymbol{l}}
\newcommand{\bn}{\boldsymbol{n}}
\newcommand{\bo}{\boldsymbol{o}}
\newcommand{\bp}{\boldsymbol{p}}
\newcommand{\bq}{\boldsymbol{q}}
\newcommand{\br}{\boldsymbol{r}}
\newcommand{\bs}{\boldsymbol{s}}
\newcommand{\bt}{\boldsymbol{t}}
\newcommand{\bu}{\boldsymbol{u}}
\newcommand{\bv}{\boldsymbol{v}}
\newcommand{\bw}{\boldsymbol{w}}
\newcommand{\bx}{\boldsymbol{x}}
\newcommand{\by}{\boldsymbol{y}}
\newcommand{\bz}{\boldsymbol{z}}
\newcommand{\bmu}{\boldsymbol{\mu}}
\newcommand{\bA}{\boldsymbol{A}}
\newcommand{\bB}{\boldsymbol{B}}
\newcommand{\bC}{\boldsymbol{C}}
\newcommand{\bD}{\boldsymbol{D}}
\newcommand{\bE}{\boldsymbol{E}}
\newcommand{\bF}{\boldsymbol{F}}
\newcommand{\bG}{\boldsymbol{G}}
\newcommand{\bH}{\boldsymbol{H}}
\newcommand{\bI}{\boldsymbol{I}}
\newcommand{\bJ}{\boldsymbol{J}}
\newcommand{\bK}{\boldsymbol{K}}
\newcommand{\bL}{\boldsymbol{L}}
\newcommand{\bM}{\boldsymbol{M}}
\newcommand{\bN}{\boldsymbol{N}}
\newcommand{\bO}{\boldsymbol{O}}
\newcommand{\bP}{\boldsymbol{P}}
\newcommand{\bQ}{\boldsymbol{Q}}
\newcommand{\bR}{\boldsymbol{R}}
\newcommand{\bS}{\boldsymbol{S}}
\newcommand{\bT}{\boldsymbol{T}}
\newcommand{\bU}{\boldsymbol{U}}
\newcommand{\bV}{\boldsymbol{V}}
\newcommand{\bW}{\boldsymbol{W}}
\newcommand{\bX}{\boldsymbol{X}}
\newcommand{\bY}{\boldsymbol{Y}}
\newcommand{\bZ}{\boldsymbol{Z}}
\newcommand{\balpha}{\boldsymbol{\alpha}}
\newcommand{\bbeta}{\boldsymbol{\beta}}
\newcommand{\bgamma}{\boldsymbol{\gamma}}
\newcommand{\bdelta}{\boldsymbol{\delta}}
\newcommand{\bepsilon}{\boldsymbol{\epsilon}}
\newcommand{\blambda}{\boldsymbol{\lambda}}
\newcommand{\bnu}{\boldsymbol{\nu}}
\newcommand{\bphi}{\boldsymbol{\phi}}
\newcommand{\bpi}{\boldsymbol{\pi}}
\newcommand{\bsigma}{\boldsymbol{\sigma}}
\newcommand{\btheta}{\boldsymbol{\theta}}
\newcommand{\bomega}{\boldsymbol{\omega}}
\newcommand{\bxi}{\boldsymbol{\xi}}
\newcommand{\bGamma}{\boldsymbol{\rho}}
\newcommand{\bDelta}{\boldsymbol{\Delta}}
\newcommand{\bTheta}{\boldsymbol{\Theta}}
\newcommand{\bLambda}{\boldsymbol{\Lambda}}
\newcommand{\bXi}{\boldsymbol{\Xi}}
\newcommand{\bPi}{\boldsymbol{\Pi}}
\newcommand{\bOmega}{\boldsymbol{\Omega}}
\newcommand{\bUpsilon}{\boldsymbol{\Upsilon}}
\newcommand{\bPhi}{\boldsymbol{\Phi}}
\newcommand{\bPsi}{\boldsymbol{\Psi}}
\newcommand{\bSigma}{\boldsymbol{\Sigma}}

\newcommand{\bzero}{\boldsymbol{0}}

\newcommand{\geom}{{\textsf{Geometric}}}
\newcommand{\otherwise}{{\text{otherwise}}}
\newcommand{\prox}{{\operatorname{prox}}}
\newcommand{\vect}{{\operatorname{vec}}}
\newcommand{\supp}{{\operatorname{supp}}}
\newcommand{\KL}{{Kurdyka-\L ojasiewicz }}
\newcommand{\dist}{{\operatorname{dist}}}
\newcommand{\rank}{{\operatorname{rank}}}
\newcommand{\trace}{{\operatorname{tr}}}
\newcommand{\diag}{{\operatorname{diag}}}
\newcommand{\dom}{{\operatorname{dom}}}
\newcommand{\iid}{{\overset{iid}{\ \sim\ }}}

\section{Introduction}\label{sec:intro}

Biclustering is an unsupervised machine learning task that seeks to partition observations into groups of similar elements while simultaneously partitioning their features in a similar fashion. That is, when the data are organized as a matrix, biclustering seeks to learn groupings in both the row and column variables. In contrast to clustering, which seeks only to group the observations (rows), biclustering also uncovers how groups of features define or differentiate these groupings. Applications of biclustering are numerous in genomics \citep{chengBiclusteringExpressionData2000, klugerSpectralBiclusteringMicroarray2003, fadasonChromatinInteractionsExpression2018}, data mining \citep{busyginBiclusteringDataMining2008} and precision medicine \citep{nezhadSUBICSupervisedBiClustering2017}.

There is a rich literature related to biclustering, but many existing methods were not designed to scale to modern datasets. Biclustering is inherently a combinatorially challenging problem due to the many possible row and column partitions to consider.  Early methods performed a greedy approach that focused on searching data matrices for submatrices of similar magnitude \citep{chengBiclusteringExpressionData2000}. Many search strategies require parametric assumptions or heuristics to maintain a reasonable runtime \citep{shabalinFindingLargeAverage2009, liuMultiscaleScanStatistic2019}. Biclustering has also been studied from the lens of matrix decomposition. Popular algorithms such as PLAID seek a ``checkerboard'' structure among rows and columns of the data as sums of outer products between membership assignments and bicluster centroids \citep{lazzeroniPlaidModelsGene2002}. Assuming an additive model, the method sequentially solves least squares problems, and was later improved upon with a search strategy \citep{turnerImprovedBiclusteringMicroarray2005}. Noting the inherent low-rank structure behind biclustered data, other approaches enforce sparsity on singular values either by manually thresholding after singular value decomposition \citep{bergmannIterativeSignatureAlgorithm2003} or post-processing based on domain knowledge \citep{klugerSpectralBiclusteringMicroarray2003}. Sparsity-inducing penalties for regularizing singular vectors include the SSVD algorithm \citep{leeBiclusteringSparseSingular2010} and BCEL \citep{zhongBiclusteringStructuredRegularized2022}.
Both methods incorporate stability selection to improve performance by refining estimated biclusters \citep{sillRobustBiclusteringSparse2011}. Although the matrix factorization lens is highly flexible, it has been reported to perform poorly when the signal-to-noise ratio is low, common in high-dimensional settings \citep{nichollsComparisonSparseBiclustering2021}.

In many applications, data such as microarray gene expression or single-cell RNA-sequencing are collected over a large number of features that overwhelm the number of available observations. In such cases, often many features bear little to no signal useful towards biclustering. Our contributions focus improving biclustering in this setting by identifying informative features that are not only important toward successfully discriminating groups (rows) among the data, but also necessary to discard noise features (columns) that otherwise detract from the quality of the biclustering. So that existing methods not suited for high-dimensional data can apply, it is common practice to apply a preprocessing step such as PCA to heuristically discard most features. While this ad hoc fix can improve clustering performance compared to no preprocessing, it risks spurious findings, such as prematurely discarding important features. This article seeks a more principled approach that selects and weighs features jointly throughout the biclustering task.

To this end, we propose a new formulation to accomplish biclustering together with feature weighing. Our method draws upon recent ideas from the convex clustering and biclustering literature, aiming to better handle high-dimensional data while benefiting from stability and theoretical guarantees established for these methods. Convex clustering \citep{hockingClusterpathAlgorithmClustering2011} formulates the clustering task as a convex optimization problem
\begin{equation}
    \min_{\bU \in \mathbb R^{n\times p}} \gamma \sum_{i < j}^n \Phi_{ij} \|\bU_{i \cdot} - \bU_{j \cdot}\|_2 + \frac{1}{2} \|\bU - \bX\|^2_F, \label{eqn:clusterpath}
\end{equation}
where $\bX$ is a data matrix of $n$ samples and $p$ features and the optimization variable $\bU$ is of the same size. 
The objective \eqref{eqn:clusterpath} can be understood as a convex relaxation of single-linkage hierarchical clustering, where a ``fusion'' penalty encourages a unique number of rows, interpreted as cluster centroids. Here, $\gamma$ is a hyperparameter controlling the strength of fusion and $\Phi_{ij} \geq 0$ are affinities to fuse various rows. This convex formulation has many attractive properties, such as a continuous solution path with respect to $(\bX, \bPhi, \gamma)$ \citep{chiSplittingMethodsConvex2015}, computational efficiency from using a sparse $\bPhi$, and solutions not sensitive to initialization. The success of convex clustering has lead to a myriad of options for optimization \citep{chiSplittingMethodsConvex2015,weylandtDynamicVisualizationFast2020, panahiClusteringSumNorms2017, sunConvexClusteringModel2021}, parallel processing routines, \citep{fodorParallelADMMbasedConvex2022, zhouScalableAlgorithmsConvex2021} and multi-view data \citep{wangIntegrativeGeneralizedConvex2021}. 

Biclustering can be cast similarly by including an additional penalty term to promote column fusion, resulting in the objective
\begin{equation}
    \min_{\bU \in \mathbb R^{n\times p}} \gamma \brk3{\sum_{i < j}^n \Phi_{ij}\|\bU_{i\cdot}-\bU_{j\cdot}\|_2 + \sum_{k < \ell}^p \Psi_{k \ell} \|\bU_{\cdot k}-\bU_{\cdot \ell}\|_2 } + \frac{1}{2} \|\bX - \bU\|^2_F. \label{eqn:cbc}
\end{equation}
Problem \eqref{eqn:cbc} can be solved efficiently using COBRA \citep{chiConvexBiclustering2017}, a proximal method that alternates between convex clustering subroutines on the rows and columns until convergence. Several improvements were subsequently proposed using compression strategies \citep{yiCOBRACFastImplementation2021}, alternating direction method of multipliers (ADMM) \citep{wangNewAlgorithmConvex2023}, and generalized ADMM \citep{dengGlobalLinearConvergence2016, weylandtSplittingMethodsConvex2019}. Like the standard clustering counterpart, the efficacy of convex biclustering has lead to generalizations including tensor versions \citep{chiProvableConvexCoclustering2020}, loss functions robust to outliers \citep{chenRobustConvexBiclustering2025}, and supervised learning \citep{nezhadSUBICSupervisedBiClustering2017}.

While convex clustering and biclustering have seen practical success, these methods also begin to struggle as the dimension of the problem increases. In high dimensions, Euclidean distance loses ability to discriminate observations \citep{aggarwalSurprisingBehaviorDistance2001}; this is exacerbated as both the norms appearing in the objectives and commonly used affinities begin to lose fidelity. Affinities have large repercussions in practice, as we explore in Section \ref{section:problem}, and as evident from Equation \eqref{eqn:cbc} which requires two sets of affinities for biclustering regularization.

To make progress, we expect that in high-dimensional settings biclustering may remain feasible when a subset of columns are salient to the task. Equations \eqref{eqn:clusterpath} and \eqref{eqn:cbc} treat all features equally in their formulation, yet it is natural to identify and downweight less informative features. Feature selection and weighing have proven successful in the context of clustering: \citet{wittenFrameworkFeatureSelection2010} proposed a framework for a variety of clustering tasks by decomposing losses into weighted averages of the column measures of fit. SparseBC \citep{tanSparseBiclusteringTransposable2014} directly generalizes the $k$-means objective function to include column centroids together with a sparsity-inducing penalty.
SC-Biclust \citep{helgesonBiclusteringSparseClustering2020} extends these ideas to biclustering in a straightforward way by applying the algorithm in stages, and then testing the fitted weights for discrimination between column contributions to row clusters. Based on $k$-means, these methods inherit some of the known drawbacks, including sensitivity to initialization, difficulty with low signal-to-noise ratios, and inability to handle non-linear separation. 

Our proposed approach seeks to achieve the advantages of feature weighting while largely preserving the stability enjoyed by convex formulations. Recently, \citet{chakrabortyBiconvexClustering2023} made a small departure from convexity in the context of clustering by augmenting \eqref{eqn:clusterpath} with weights: 
\begin{equation}
  \begin{split}
    \min_{\bU \in \mathbb R^{n\times p}, \bw \in \mathbb R^p}\ & \gamma \sum_{i < j} \Phi_{ij} \|\bU_{i \cdot} - \bU_{j \cdot}\|^2_2 + \frac{1}{2} \sum_{\ell=1}^p (w_\ell^2 + \lambda w_\ell) \|\bU_{\cdot \ell} - \bX_{\cdot \ell}\|^2_2, \\
    \text{ subject to } & \sum_{\ell=1}^p w_\ell = 1, w_\ell \geq 0\label{eqn:bcc}.
  \end{split}
\end{equation}
With $\bw$ as an additional optimization variable, \eqref{eqn:bcc} is now biconvex: the optimization problem in either $\bU$ or $\bw$ is convex when the other is held fixed. We aim to overcome the difficulties of convex biclustering in high dimensions by changing the measure-of-fit in \eqref{eqn:cbc} to include learning column weights similar to \eqref{eqn:bcc} in a practical and theoretically sound framework.  To this end, we contribute a novel optimization scheme based on Proximal Alternating Linearized Minimization \citep{bolteProximalAlternatingLinearized2014}. Not only do we establish theoretical guarantees including finite-sample bounds but find that this formulation allows for existing efficient algorithms to be leveraged in solving the resulting subproblems. Moreover, this method does not require squared penalties on the fusion terms as in \eqref{eqn:bcc}, leading to better merging and more robustness to outliers as a result. A two-stage procedure is developed to tune the balance of the fusion and sparsity hyperparameters by principled approaches. In addition, we generalize statistical guarantees established for convex clustering.

In Section \ref{section:problem}, we define our problem statement and objective function while emphasizing the utility and importance of the affinities, $\bPhi$ and $\bPsi$. In Section \ref{section:methods}, we present the proposed Biconvex Biclustering (BCBC) method and an adaptive variant that updates the data affinities during data fitting. Section \ref{section:theory} shows a finite-sample bound for solutions of our objective function that allows for weighted connected affinity graphs, a generalization recently seen in \citet{linLowRank}, an application of convex clustering to low-rank matrix data. In Section \ref{section:simulations}, we show empirically that Adaptive BCBC fits consistently recover biclusters from simulated data, while weighing the features appropriately for their contribution to the biclustering. Finally, in Section \ref{section:real_data} we consider a case study on lymphoma microarray data and discover a variety of meaningful biclusters. Throughout this work, upper-case bold letters will denote matrices, while indexing with $\bU_{i\cdot}$ and $\bU_{\cdot \ell}$ signify the $i$th row and $\ell$th column of $\bU$, respectively. Lower-case bold letters denote vectors, with individual elements in nonbolded text, e.g. $w_{\ell}$ is the $\ell$th element of $\bw$. Symbols $n$ and $p$ will always refer to the number of rows (samples) and columns (features) of a matrix, respectively. In addition, $\|\cdot\|_F$ refers to the Frobenius norm, $\|\cdot\|_0$ the number of nonzero entries and occasionally we write the weighted norm as $\|\bX\|^2_{\bw} = \sum_{\ell = 1}^p w_\ell \|\bX_{\cdot \ell}\|_2^2$.

\section{Biconvex Formulation} \label{section:problem}

Let $\bX \in \mathbb R^{n\times p}$, such that there are $n$ samples of $p$ features. We seek to perform biclustering with penalty terms encouraging both column fusion and feature weighing, via minimizing the objective function 
\begin{equation} 
    \begin{split}
    F(\bU, \bw) & = \gamma \brk3{\sum_{i < j}^n \sqrt{\Phi_{ij}}\|\bU_{i\cdot}-\bU_{j\cdot}\|_2 + \sum_{k < \ell}^p \sqrt{\Psi_{k \ell} }\|\bU_{\cdot k}-\bU_{\cdot \ell}\|_2 } + \frac{1}{2} \|\bX - \bU\|^2_{\bw^2 + \lambda \bw}, \label{eqn:objective} \\
    \text{subject to } & \ \sum_{\ell = 1}^p w_\ell = 1, \quad w_\ell \geq 0.
    \end{split}
\end{equation}
The first term is convex but non-smooth with respect to $\bU$, composed of two fusion terms that encourage rows and columns, respectively, to coalesce toward shared centroids. The second term is biconvex in the variables $(\bU, \bw)$ and measures goodness-of-fit, weighted by how useful each feature is toward biclustering. The tuning parameter $\gamma>0$ modulates the tradeoff. The weights $\bw$ are optimized jointly along with $\bU$, allowing feature selection to occur simultaneously with biclustering. The weighted norm is inspired by \citet{wittenFrameworkFeatureSelection2010}, which avoids degenerate solutions under only an $\ell_1$ term by additionally including an $\ell_2$ constraint or penalty; \citet{chakrabortyBiconvexClustering2023} operate under a similar logic. The \textit{affinity terms} $\bPhi$ and $\bPsi$ are square-rooted for notational convenience of our theory in Section \ref{section:theory}.

In practice, the choices for affinities $\bPhi$ and $\bPsi$ have a dramatic effect on the solutions to \eqref{eqn:objective} \citep{chiWhyHowConvex2025}. Inspecting $\bPhi$, we see that only pairs $(i, j)$ where $\Phi_{ij} > 0$ encourage fusion between rows $i$ and $j$, eventually coalescing them to a single centroid as $\gamma$ increases. A higher number of \textit{informative} $\Phi_{ij} > 0$ will increase the solution quality, but including excess uninformative pairs both increases computational burden and may decrease solution quality. \citet{hockingClusterpathAlgorithmClustering2011} suggested setting affinities using the data with $\Phi_{ij} = \exp(-\tau \|\bX_{i\cdot} - \bX_{j\cdot}\|^2_2)$, for some hyperparameter $\tau > 0$. Empirical evidence shows this is a reasonable choice. \citet{chiSplittingMethodsConvex2015} expanded this idea further by suggesting setting all $\Phi_{ij} = 0$ when rows $i$ and $j$ are not $k$-nearest neighbors. Doing so reduces the total arithmetic calculations of fusion terms from $O(n^2 p)$ to $O(p \|\bPhi\|_0)$, a recommendation that is 
followed in many subsequent papers.  Equivalently, the affinity matrix $\bPhi$ can be treated as an adjacency matrix for some weighted simple graph, $G = (V, E, W)$, where $|V| = n, |E| = \|\bPhi\|_0 / 2$, and $W(i, j) = \Phi_{ij}$ represent each data point, edge, and resulting weight, respectively. We make use of this connection in our theoretical analysis in Section \ref{section:theory}.

\section{Methods} \label{section:methods}
To enable estimation, we note the objective can be rewritten as
\begin{equation}
    F(\bU, \bw) = f(\bU) + g(\bw) + H(\bU, \bw), \label{eqn:objective_short}
\end{equation}
where $f$ represents the fusion penalties, $g$ is the $0/\infty$ indicator function for whether $\bw$ satisfies the constraints in \eqref{eqn:objective}, and $H$ is the weighted loss. Exploiting the structure of \eqref{eqn:objective_short} will be key to efficient optimization, as $f$ and $g$ are both proper, lower semicontinuous and convex, while $H$ is smooth and biconvex. Moreover, the convex terms $f$ and $g$ are parameter separated, suggesting they may be handled individually within an alternating procedure. Because the biconvex term $H$ is smooth and can be split into blocks that mirror the separability of $f$ and $g$, it can also be incorporated within block alternating schemes. 

With these aspects in mind, we solve \eqref{eqn:objective_short} using Proximal Alternating Linearized Minimization (PALM) \citep{bolteProximalAlternatingLinearized2014}, an iterative scheme that solves simpler optimization subproblems in each block repeatedly. The subproblems consist of optimizing a quadratic approximation at a current iterate plus a linearization from the smooth term. In our setting, these approximations about $(\bU^0, \bw^0)$ take the form
\begin{align}
    F_{\bw}^\nu(\bU; \bU^0) & := f(\bU) + (\bU - \bU^0)^\top \nabla_{\bU} H(\bU^0, \bw) + \frac{\nu}{2} \|\bU - \bU^0\|^2_F,  \label{eq:palm1} \\
    F_{\bU}^\nu(\bw; \bw^0) & := g(\bw) + (\bw - \bw^0)^\top \nabla_{\bw} H(\bU, \bw^0) + \frac{\nu}{2} \|\bw - \bw^0\|^2_2. \label{eq:palm2}
\end{align}
Basic manipulations reveal that solutions to these subproblems are related to the proximal operator \citep{parikhProximalAlgorithms2014}
\begin{equation}
    \prox_{\nu,h}(u) := \argmin_{u'} h(u') + \frac{\nu}{2} \|u - u'\|_2^2.
\end{equation}
In particular, $\bw \in \argmin F_{\bU}^\nu(\cdot; \bw^0) \iff \bw \in \prox_{\nu, g}\brk[s]2{\bw^0 - \frac{1}{\nu} \nabla_{\bw} H(\bU, \bw^0)}$, with an analogous result for $\bU \in \argmin F_{\bw}^\nu(\cdot; \bU^0)$. We see that PALM iteratively optimizes objectives of the form \eqref{eqn:objective_short} by handling non-smooth terms via proximal smoothing. In particular, the PALM updates reduce our biclustering task to subproblems with known, well-behaved solutions as summarized in the following proposition.

\begin{proposition} The minimizers of subproblems \eqref{eq:palm1}, \eqref{eq:palm2} applied to 
the objective function \eqref{eqn:objective_short} are given by solving a biconvex clustering problem and projecting onto a weighted simplex, respectively.
\end{proposition}

\begin{proof} We begin by writing the derivatives of $H$
\begin{align}
    \nabla_{\bU_{\cdot \ell}} H(\bU, \bw) & = (w_\ell^2 + \lambda w_\ell) (\bU_{\cdot \ell} - \bX_{\cdot \ell}), \\
    \nabla_{w_\ell} H(\bU, \bw) & = (w_\ell + \lambda/2) D_\ell(\bU),
\end{align}
where $D(\bU)\colon \mathbb R^{n\times p} \mapsto \mathbb R^p$ with $D_\ell(\bU) = \|\bX_{\cdot \ell} - \bU_{\cdot \ell}\|^2_2$. The first subproblem $\prox_{\nu_1, f}\brk[s]2{\bU^0 - \frac{1}{\nu_1} \nabla_{\bU} H(\bU^0, \bw)}$ with $\bw$ fixed starting from an initial point $\bU^0$ is given by
\begin{equation}
\begin{split}
    & \argmin_{\bU \in \mathbb R^{n \times p}} \frac{\gamma}{\nu_1} \brk3{\sum_{i < j}^n \sqrt{\Phi_{ij}}\|\bU_{i\cdot}-\bU_{j\cdot}\|_2 + \sum_{k < \ell}^p \sqrt{\Psi_{k \ell}} \|\bU_{\cdot k}-\bU_{\cdot \ell}\|_2 } \\ & \quad + \frac{1}{2} \norm2{\bU^0 - \frac{1}{\nu_1} (\bU^0 - \bX) \diag(\bw^2 + \lambda \bw) - \bU}_F^2.
    \label{eqn:U_step}
\end{split}
\end{equation}
We recognize this objective can be viewed equivalently as a convex biclustering task \eqref{eqn:cbc}, where the data matrix is $\bU^0 - \frac{1}{\nu_1} (\bU^0 - \bX) \diag(\bw^2 + \lambda \bw)$. This subproblem can thus be solved directly using COBRA or other viable alternatives such as \citet{yiCOBRACFastImplementation2021} or \citet{wangNewAlgorithmConvex2023}.

The remaining subproblem, $\prox_{\nu_2, g}\brk[s]2{\bw^0 - \frac{1}{\nu_2} \nabla_{\bw} H(\bU, \bw^0)}$, with $\bU$ fixed and an initial point $\bw^0$ is given by 
\begin{equation}
  \argmin_{\substack{\bw\colon \sum w_\ell = 1 \\ w_\ell \geq 0 }} \sum_{\ell = 1}^p \brk[s]2{w^0_\ell - \frac{1}{\nu_2}\brk2{w^0_\ell + \frac{\lambda}{2}}\|\bX_{\cdot \ell} - \bU_{\cdot \ell}\|^2_2 - w_\ell }^2. \label{eqn:w_step}
\end{equation}
This amounts to orthogonal projection of $\bw^0 - \frac{1}{\nu_2}\brk1{\bw_{0} + {\lambda} \mathbf 1_p/2 } \odot D(\bU) $ onto the simplex, where $\odot$ is the Hadamard product. See \citet{chenProjectionSimplex2011} and \citet{condatFastProjectionSimplex2016a} for details on fast projection onto the simplex. \end{proof} 

In practice, we declare convergence using a relative error stopping criterion on the iterates $\bU^k$, followed by a single block-optimization step on the final $\bw^k$ with $\bU$ fixed
(see Theorem 1 of \citet{chakrabortyBiconvexClustering2023}). A pseudocode summary is provided in Algorithm \ref{alg:bcbc}.

\begin{algorithm}[!ht]
\caption{Optimizing Biconvex Biclustering via PALM} 
\label{alg:bcbc}
\hspace*{0.00in} {\bf Input:} Data points $\bX$, $\bw^0 \in \mathbb R^p, \sum_{\ell=1}^p w^0_{\ell} = 1, \nu^- > 0$
\begin{algorithmic}[1]
    \State Initialize $k \leftarrow 0$, $\bU^0 \leftarrow \bX$.
    \While{$\bU$ has not converged}
        \State $\nu_1 \leftarrow \max[1, 2 L_1(\bw^k)]$ from \eqref{eqn:lip_constants}
        \State $\bU^{k+1} \leftarrow \prox_{\nu_1, f}[\bU^k - \nabla_{\bU} H(\bU^k, \bw^k) / \nu_1]$ (Convex biclustering via \eqref{eqn:U_step})
        \State $\nu_2 \leftarrow \max[\nu^-, 2 L_2(\bU^{k+1})]$ from \eqref{eqn:lip_constants}
        \State $\bw^{k+1} \leftarrow \prox_{\nu_2, g}[\bw^k - \nabla_{\bw} H(\bU^{k+1}, \bw^k) / \nu_2]$ (Project to simplex via \eqref{eqn:w_step})
        \State $k \leftarrow k + 1$
    \EndWhile
    \State $\hat \bw \leftarrow \argmin_{\bw} g(\bw) + H(\bU^k, \bw) $ (Coordinate descent step on $\bw$ block)
    \State \Return $\bU^k, \hat \bw$.
\end{algorithmic}
\end{algorithm}

\subsection{Convergence}\label{section:convergence_details}

To guarantee that Algorithm \ref{alg:bcbc} converges to the limit points of the objective, PALM requires the following assumptions:
\begin{itemize}
    \item \textbf{Assumption 1}: $f, g$ are both proper and lower semicontinuous. $H$ is a $C^1$ function.
    
    \item \textbf{Assumption 2(i)}: $\inf F > -\infty, \inf f > -\infty, \inf g > -\infty$.
    
    \item \textbf{Assumption 2(ii)}: For a fixed $\bw$, the function $H(\bU,  \bw)$ has a partial gradient $\nabla_{\bU} H(\bU, \bw)$ which is globally Lipschitz, i.e.
    \begin{equation}
        \|\nabla_{\bU} H(\bU^1; \bw) - \nabla_{\bU} H(\bU^2; \bw)\|_F \leq L_1(\bw) \|\bU^1 - \bU^2\|_F. \label{eqn:assumption2iix}
    \end{equation}
    Similarly, for a fixed $\bU$, the function $H(\bU, \bw)$ has a partial gradient $\nabla_{\bw} H(\bU, \bw)$ which is globally Lipschitz, i.e.
    \begin{equation}
        \|\nabla_{\bw} H(\bU; \bw^1) - \nabla_{\bw} H(\bU; \bw^2)\|_2 \leq L_2(\bU) \|\bw^1 - \bw^2\|_2. \label{eqn:assumption2iiy}
    \end{equation}
It is not hard to see that these are all satisfied by \eqref{eqn:objective_short}. In particular, \eqref{eqn:assumption2iix} and \eqref{eqn:assumption2iiy} are satisfied with 
    \begin{equation}
        L_1(\bw) = \max_\ell w_\ell^2 + \lambda w_\ell; \qquad     L_2(\bU) = \sqrt{\sum_{\ell=1}^p \|\bX_{\cdot \ell} - \bU_{\cdot \ell}\|_2^4 } \label{eqn:lip_constants}.
    \end{equation}
Calculation of these Lipschitz constants affects the dynamic step sizes throughout PALM. When substituting $\nu$ with equations \eqref{eqn:lip_constants}, the subproblems \eqref{eqn:U_step} and \eqref{eqn:w_step} gain some intuition as updates to one block are based on normalization of the residuals based off the other block. For example, step 6 of Algorithm \ref{alg:bcbc} is an orthogonal projection of $\bw^k - \brk1{\bw^k + {\lambda} \mathbf 1_p/2 } \odot \frac{D(\bU^{k+1})}{L_2(\bU^{k+1})} $, where the right term of the Hadamard product is a unit vector.

    \item \textbf{Assumption 2(iii)}: The Lipschitz constants from \eqref{eqn:lip_constants} of the PALM iterates, $(\bU^k, \bw^k)$ is bounded above 0 and below infinity. This is satisfied by
    \begin{align}
        \inf_k L_1(\bw^k) & \geq \frac{1}{p^2} + \frac{\lambda}{p}, &
        \sup_k L_1(\bw^k) & \leq 1+\lambda, \\
        \inf_k L_2(\bU^k) & > 0, &
        \sup_k L_2(\bU^k) & < \infty.
    \end{align}
    There are clear bounds of $L_1(\bw)$ for all $\bw$ in the domain of $F$. For $L_2(\bU^k)$ we can create an artificial lower bound by setting $L'_2(\bU^k) = \max[c, L_2(\bU^k)]$ as recommended by \citet{bolteProximalAlternatingLinearized2014}. Although there is not an explicit upper bound, in practice the iterates $\bU^k$ need to approximate $\bX$ because the objective function decreases throughout iteration, allowing us to assume $L_2(\bU^k)$ will never be too large.

    \item \textbf{Assumption 2(iv)}: $\nabla H$ is Lipschitz continuous on bounded subsets of $\mathbb R^{n \times p} \times \mathbb R^p$:
    \begin{equation}
        \|\nabla H(\bU^1, \bw^1) - \nabla H(\bU^2, \bw^2)\|_2 \leq M\|(\bU^1 - \bU^2, \bw^1 - \bw^2)\|_2,
    \end{equation}
    where $M > 0$. As noted in Remark 3 of \citet{bolteProximalAlternatingLinearized2014}, this is satisfied because $H$ is twice continuously differentiable.

\end{itemize}

We see it is sufficient to take $\nu_1 = c_1 L_1(\bw^k)$ and $\nu_2 = c_2 L_2(\bU^k)$ where $c_1, c_2 > 1$ in order to guarantee convergence of Algorithm \ref{alg:bcbc}. Notably $\nu_1$ will rescale the hyperparameter $\gamma$ when solving convex biclustering subproblems (see \eqref{eqn:U_step}), so for numerical stability and ensuring problems are on the same scale, we force $\nu_1 = 1$ for almost all iterations by having $c_1 = \max[1/L_1(\bw^k), 2]$. The Lipschitz constant from \eqref{eqn:lip_constants} is almost always less than 1 in high-dimensional cases as both $\lambda$ and $w_\ell$ are on the scale of $1 / p$. For $\nu_2$, we choose $c_2 = 2$ for simplicity. 

The above assumptions ensure Algorithm \ref{alg:bcbc} converges, but if the objective function satisfies the \KL property then we can ensure limit points are also stationary points of the objective.

\begin{proposition} \label{proposition:FisKL}
    $F(\bU, \bw)$ satisfies the \KL property everywhere. As a result, Algorithm \ref{alg:bcbc} converges to a stationary point of \eqref{eqn:objective}.
\end{proposition}

We defer relevant definitions and proofs to Section \ref{appendix:kl} of the appendix.

\subsection{Adaptive BCBC}
The objective function for biconvex biclustering \eqref{eqn:objective} can be modified slightly so that the affinities can be learned in a data-adaptive fashion. Consider
\begin{equation}  
    \begin{split} 
        F (\bU, \bw, \bPhi, \bPsi) & = \gamma \brk3{\sum_{i < j}^n \sqrt{\Phi_{ij}}\|\bU_{i\cdot}-\bU_{j\cdot}\|_2 + \sum_{k < \ell}^p \sqrt{\Psi_{k \ell} }\|\bU_{\cdot k}-\bU_{\cdot \ell}\|_2 } + \frac{1}{2} \|\bX - \bU\|^2_{\bw^2 + \lambda \bw},  \\
        \text{subject to } & \ \sum_{\ell = 1}^p w_\ell = 1, \quad w_\ell \geq 0, \quad \bPhi \geq 0, \quad \bPsi \geq 0.
    \end{split} \label{eqn:objective_adapt}
\end{equation}
As discussed in the introduction, $\bPhi$ and $\bPsi$ are essential for performing the necessary fusion to have viable clusters.
When $\bPhi$ and $\bPsi$ are calculated from the raw data, they may suffer from the large number of noise features just as the original biconvex clustering objective struggles. It is therefore intuitive to update the affinities while fitting BCBC, as the learned weighted feature space can alleviate this issue. An adaptive variant can be obtained immediately upon inserting an additional optional update in Algorithm \ref{alg:bcbc}, performing $k$-nearest neighbors on updated rows and columns of $\bU$ under a Gaussian kernel. These updates follow the intuition in \citep{chiConvexBiclustering2017}, and are given by  
\begin{align}
    \tilde \Phi^k_{ij} & = \mathbbm{1}(\text{row $i$ is NN with row $j$})  \exp\brk1{- \tau \|\bU^k_{i\cdot} - \bU^k_{j\cdot}\|^2_2/p }, \label{eqn:normalized_fusions1} \\
    \tilde \Psi^k_{ \ell m} & = \mathbbm{1}(\text{column $\ell$ is NN with column $m$}) \exp\brk1{- \tau \|\bU^k_{\cdot \ell} - \bU^k_{\cdot m}\|^2_2 / n}, \\
    \sqrt{\Phi_{ij}^k} & = \frac{\tilde \Phi^k_{ij}}{\sqrt{p}\sum_{i',j'} \tilde \Phi^k_{i'j'}}, \\
    \sqrt{\Psi^k_{\ell m}} & = \frac{\tilde \Psi^k_{\ell m}}{\sqrt{n}\sum_{\ell', m'} \tilde \Psi^k_{\ell' m'}}. \label{eqn:normalized_fusions2}
\end{align}

We also note that using approximate nearest neighbor methods such as HNSW \citep{malkovEfficientRobustApproximate2018} in place of exact $k$-NN can speed up Algorithm \ref{alg:bcbc_adapt} with essentially no loss of performance.

\begin{algorithm}[!ht]
\caption{Optimizing Adaptive Biconvex Biclustering via PALM and KNN} 
\label{alg:bcbc_adapt}
\hspace*{0.00in} {\bf Input:} Data points $\bX$, $\bw^0 \in \mathbb R^p, \sum_{\ell=1}^p w^0_{\ell} = 1, \nu^- > 0, \bPhi^0, \bPsi^0$
\begin{algorithmic}[1]
    \State Initialize $k \leftarrow 0$, $\bU^0 \leftarrow \bX$.
    \While{not converged}
        \State $\nu_1 \leftarrow \max[1, 2 L_1(\bw^k)]$ from \eqref{eqn:lip_constants}
        \State $\bU^{k+1} \leftarrow \prox_{\nu_1, f}[\bU^k - \nabla_{\bU} H(\bU^k, \bw^k) / \nu_1]$ (Convex biclustering via \eqref{eqn:U_step})
        \State $\nu_2 \leftarrow \max[\nu^-, 2 L_2(\bU)]$ from \eqref{eqn:lip_constants}
        \State $\bw^{k+1} \leftarrow \prox_{\nu_2, g}[\bw^k - \nabla_{\bw} H(\bU^{k+1}, \bw^k) / \nu_2]$ (Project to simplex via \eqref{eqn:w_step})
        \State $(\bPhi^{k+1}, \bPsi^{k+1}) \leftarrow$ GKNN affinities with $\bU^{k+1}$ from equations \eqref{eqn:normalized_fusions1}--\eqref{eqn:normalized_fusions2}.
        \State $k \leftarrow k + 1$
    \EndWhile
    \State $\hat \bw \leftarrow \argmin_{\bw} g(\bw) + H(\bU^k, \bw) $ (Coordinate descent step on $\bw$ block)
    \State \Return $\bU^k, \hat \bw$.
\end{algorithmic}
\end{algorithm}

\subsection{Hyperparameter Selection}\label{section:cv}

Following \citet{tanSparseBiclusteringTransposable2014}, we recommend a two-stage tuning approach with an initial stage to recommend a value for $\gamma$, and a second stage to evaluate fits for varying values of $\lambda$. Some computational burden is relieved as not every possible pair of $(\gamma, \lambda)$ is fit while allowing each hyperparameter to be optimized by principled approaches.

\paragraph*{Tuning \texorpdfstring{$\gamma$}{gamma}} Similar to \citet{chiConvexBiclustering2017}, we create a hold-out set of data points and perform a missing data fit to tune $\gamma$. Let $\mathcal I$ be a set of observed indices. By default, we recommend $|\mathcal I| \approx 0.85 \cdot np$ and chosen independently at random; however, tuning via the hold-out problem may be more effective if $|\mathcal I|$ decreases as the number of uninformative columns increases. The hold-out problem is written as:
\begin{align}
    \begin{split} \label{eqn:missing_objective}
        \tilde F_{\gamma, \lambda}(\bU, \bw) & := \gamma \brk3{\sum_{i < j}^n \sqrt{\Phi_{ij}}\|\bU_{i\cdot}-\bU_{j\cdot}\|_2 + \sum_{k < \ell}^p \sqrt{\Psi_{k \ell}} \|\bU_{\cdot k} - \bU_{\cdot \ell}\|_2 } \\ & \quad + \frac{1}{2}\sum_{\ell=1}^p (w_\ell^2 + \lambda w_\ell) \sum_{i = 1}^n \begin{cases}
            (U_{i\ell} - X_{i\ell})^2 & (i, \ell) \in \mathcal I \\
            0 & \otherwise
        \end{cases}, \\
    \text{subject to } & \ \sum_{\ell = 1}^p w_\ell = 1, \quad w_\ell \geq 0.
    \end{split}
\end{align}
See Section \ref{appendix:hold-out} of the appendix for details on efficiently optimizing \eqref{eqn:missing_objective}. In order to choose an optimized $\gamma^*$ we recommend solving the hold-out problem, and calculating the hold-out sum of squared errors for each $\gamma$:
\begin{align}
    \bU^*(\gamma) & = \argmin_{\bU} \tilde F_{\gamma, 0}(\bU, \bw), \\
    \gamma^* & = \argmin_{\gamma \in \{\gamma_1, \ldots, \gamma_m\}} \sum_{(i, j) \in \mathcal I^c} [U^*(\gamma)_{ij} - X_{ij}]^2.
\end{align}
Notably, we set $\lambda = 0$ for this tuning stage to avoid overzealously removing features, since this is reserved for the following stage. The value of $\lambda$ used has little effect on the hold-out sum of squared errors, since any hold-out value is in a column either:

\begin{enumerate}
    \tightlist
    \item Useful for clustering, in which case fits should avoid erroneously assigning zero weight and an appropriate value of $\gamma$ should estimate the hold-out value.
    \item Not useful for clustering, in which case whether or not a fit assigns zero weight, any value of $\gamma$ or $\lambda$ will have equal difficulty in estimating the hold-out value.
\end{enumerate}

\paragraph*{Tuning \texorpdfstring{$\lambda$}{lambda}} Though $\lambda$ is more directly related to the columns of $\bU$ as it implies a weighted norm on the feature space, its value impacts both row clusters as well as column clusters. It is therefore natural to tune via a criterion that directly manages the tradeoff between goodness-of-fit and model complexity as a function of $\lambda$. We advocate a version of the extended Bayesian Information Criteria (eBIC) \citep{chenEXTENDEDBICSMALLnLARGEP2012}, which has proven effective for balancing complexity in other biclustering methods even when the loss is not a likelihood proper \citep{leeBiclusteringSparseSingular2010, chiProvableConvexCoclustering2020, tanSparseBiclusteringTransposable2014}. \citet{tanStatisticalPropertiesConvex2015} provides justification for using the number of unique centroids as the degrees of freedom for a BIC calculation while \citet{chiProvableConvexCoclustering2020} shows strong empirical performance with a similar idea. 
The eBIC criterion we use is:
\begin{equation}
    \text{eBIC} = np \log \brk3{\frac{\|\bU^\star(r_1, r_2) - \bX\|^2_F}{np}} + 2 \log(np) \times \hat {df}(r_1, r_2), \label{eqn:eBIC}
\end{equation}
where $\hat {df}(r_1, r_2)$ is the number of unique biclusters and $\bU^\star(r_1, r_2)$ is defined
\begin{equation}
\begin{split}
    \bU^\star_{ik}(r_1, r_2) & = \operatorname{mean}(\{X_{j\ell} \colon \bU^\star_{i\cdot}(r_1) = \bU^\star_{j\cdot}(r_1), \bU^\star_{\cdot k}(r_2) = \bU^\star_{\cdot \ell}(r_2)\}), \quad (\bw_k > 0); \\
    \bU^\star_{ik}(r_1, r_2) & = \operatorname{mean}(\{X_{j\ell} \colon \hat \bw_\ell = 0\}), \quad (\bw_k = 0). 
\end{split}\label{eqn:U-star}
\end{equation}
Calculating $\bU^\star$ is a computationally cheap operation,
and we adopt a standard approach to assign biclusters from the solutions $\hat{\bU}, \hat{\bw}$ following \citet{chiConvexBiclustering2017}, declaring two rows to belong to the same cluster if the weighted distance between them is smaller than a threshold based on the standard deviation of all pairwise weighted distances. Column clusters are handled analogously, with complete details included Section \ref{appendix:cluster-assignments} of the appendix.

\section{Finite-Sample Bounds for Biconvex Biclustering} \label{section:theory}

We now describe finite-sample bounds on the prediction error that apply to local optima of \eqref{eqn:objective}, assuming independent sub-Gaussian errors and fixed connected affinity graphs. Even though calculating the global solution to \eqref{eqn:objective} is difficult due to non-convexity, any local optima, e.g. outputs from Algorithm \ref{alg:bcbc}, will still retain desirable theoretical properties. 
In contrast to previous results in the convex clustering literature, which assume $\Phi_{ij} = 1$ for all $i \neq j$, we allow $\bPhi$ to be the adjacency matrix of any fixed connected graph. The link between performance and affinity graph structure is summarized in Theorem \ref{thm:main} in generality by the algebraic connectivity, characterized by the smallest nonzero eigenvalue of the graph Laplacian. This theoretical support better aligns with what is done in practice, as fully uniform affinity weights lack both sparsity and give a less robust fusion penalty. Connectivity is a mild assumption: if the underlying affinity graphs had multiple connected components, one may apply a divide and conquer approach and run multiple instances of our method on each connected component. The following result is a generalization of some results from \cite{chakrabortyBiconvexClustering2023}, but now with respect to any learned weighted norm and more practical affinities.

\begin{theorem}\label{thm:main}
    Suppose that $\bX = \bU + \bE$, where $\bE \in \mathbb R^{n \times p}$ is composed of independent sub-Gaussian random variables with mean zero and variance $\sigma^2$. Let $\bPhi$ and $\bPsi$ be weighted adjacency matrices of connected graphs, independent of $\bE$, with $n$ and $p$ vertices, respectively. Let $\hat \bU$ and $\hat \bw$ be  local minimizers of \eqref{eqn:objective}.
    If $\frac{\gamma}{np} > \frac{\sigma(1+\lambda)}{\sqrt{np}} \max \brk3{\sqrt{\frac{\log(p \|\bPhi\|_0)}{n \eta(\bPhi)}}, \sqrt{\frac{\log(n \|\bPsi\|_0)}{{p \eta(\bPsi)}}}} $ then
    \begin{equation}
    \begin{split}
        \frac{1}{2np} \|\bU - \hat \bU\|^2_{\hat \bw^2 + \lambda \hat \bw}
        & \leq \frac{\sigma^2 (1+\lambda)}{2} \brk[s]3{\frac{1}{n} + \frac{1}{p} + \sqrt{\frac{\log(np)}{n p^2}} + \sqrt{\frac{\log(np)}{n^2 p}}} \\
            & \quad + \frac{2\gamma}{np} \brk3{ \sum_{i < j} \sqrt{\Phi_{ij}} \|\bU_{i \cdot} - \bU_{j \cdot}\|_2 + \sum_{k < \ell} \sqrt{\Psi_{k \ell}} \|\bU_{\cdot k} - \bU_{\cdot \ell}\|_2},
    \end{split} \label{eqn:mainthm}
    \end{equation}
    holds with probability at least $1 - 2 \exp\brk[c]3{-C \min \brk[s]2{\log(np),  \sqrt{\min(n, p) \log(np)}}} - \frac{8}{np}$, where $\eta(\bA)$ is the algebraic connectivity of the graph formed by the adjacency matrix $\bA$, $\|\bA\|_0$ is the number of nonzero entries (edges), and $C > 0$ is some constant.
\end{theorem}

Theorem \ref{thm:main} implies consistency with respect to the learned weighted norm for local solutions of \eqref{eqn:objective}. 
The terms in \eqref{eqn:mainthm} involving $\bPhi$ and $\bPsi$ provide insight into how deterministic choices of affinity graphs affect clustering quality of solutions to \eqref{eqn:objective}. Provided connectivity is maintained, any value where $\Phi_{ij} > 0$ and $\bU_{i\cdot} = \bU_{j \cdot}$, will, at minimum, decrease the right-hand side of \eqref{eqn:mainthm} by increasing the overall algebraic connectivity of the row affinity graph while not increasing the oracle terms. An analogous result holds for $\bPsi$ and column clustering. Furthermore these terms describe how poorly specified affinities, e.g. $\Phi_{ij} > 0$ but $\bU_{i\cdot} \neq \bU_{j \cdot}$, may reduce solution quality. Theorem \ref{thm:main} gives theoretical justification for the empirical finding in the literature that solution quality in convex clustering and biclustering are highly dependent on the quality of affinities. Although we have a notion of a true underlying $\bU$, we do not impose a notion of a ground truth $\bw$. Because Theorem \ref{thm:main} applies to any local solution, any choice of $\hat \bw$ will have a theoretical notion of prediction consistency for its corresponding $\hat \bU$.

\section{Empirical Performance} \label{section:simulations}

To show the efficacy of BCBC, we perform a variety of experiments on data simulated with a bicluster structure and an additional set of features that are uninformative to either clustering mode. We evaluate the fitted biclusters with the true underlying biclusters via the Adjusted Rand Index (ARI), a measure of clustering quality that adjusts for expected performance under random partitions. In addition to testing biclustering performance on the set of true biclusters, we also test the ability of the methods to detect informative features. 
The fitted weights from solutions to BCBC imply a total ordering on the utility of each feature for clustering. Weights may be sparse by explicitly zeroing out uninformative features or be dense, giving a small subset of features a high weight. To unify analysis of these possible fits over many simulations, we use Area under the ROC Curve (AUC) to test the viability of the fitted weights if used in a binary classifier for predicting if a feature is truly informative. An AUC of one indicates a perfect classifier, i.e. all weights for true features are higher than all weights for untrue features. If fitted weights have erroneous zeros for true features or have untrue features with higher weights than true features, the AUC will decrease accordingly. For non-BCBC related methods, we treat feature detection in a binary sense, where any features that are included in any bicluster are deemed informative.

We perform two simulation studies varying the signal-to-noise ratio. The first study evaluates performance in the presence of uninformative features by varying the number of pure noise columns, while holding the number of informative features fixed. The second study uses a fixed number of uninformative features, but varies both the variance of the noise and the size of the informative data matrix to test consistency of the algorithms. The data generating mechanism mirrors the design in \citet{tanSparseBiclusteringTransposable2014} for fair comparisons. We simulate datasets with 25 total biclusters. First, a bicluster center $\mu_{ij}$, is generated from a $\text{Unif}(-10, 10)$ distribution for $i=1,\ldots,5$ and $j=1,\ldots,5$. Then for each entry of $\bX \in \mathbb R^{n\times p}$, a bicluster center is chosen uniformly at random from $\mu$. Afterwards, an additional $p_{\text{extra}}$ columns of 0 are appended to $\bX$, and  independent Gaussian noise with mean 0 and variance $\sigma^2$ is added to all entries of $\bX$. Finally, both the rows and columns are shuffled and the columns are scaled to have unit variance.

All tuning for BCBC related methods is performed as described in Section \ref{section:cv}. We compare performance of our three related algorithms: Adaptive BCBC, optimization of \eqref{eqn:objective_adapt} with Algorithm \ref{alg:bcbc_adapt}; Approximate Adaptive Nearest Neighbor (ANN) BCBC, similar to Adaptive BCBC but using HNSW \citep{malkovEfficientRobustApproximate2018} for faster nearest-neighbor calculations; and BCBC, optimization of \eqref{eqn:objective} with Algorithm \ref{alg:bcbc}. In addition, we compare with a variety of other biclustering methods: BCEL \citep{zhongBiclusteringStructuredRegularized2022}, COBRA \citep{chiConvexBiclustering2017}, SC-Biclust \citep{helgesonBiclusteringSparseClustering2020}, and SparseBC \citep{tanSparseBiclusteringTransposable2014}. All tuning for these methods was done as recommended by their respective papers or software packages. Code for running BCBC, the simulated and real data experiments can be found online\footnote{\url{https://github.com/SamGRosen/BCBC/}}.

\subsection{Varying Uninformative Features}\label{section:simulation1}

For Simulation 1, we run 16 trials for all values of $p_{\text{extra}} \in \{100k \colon k = 0,\ldots,9\}$ with $\sigma = 8, n = 200$, and $p = 200$. Simulation 1 strictly tests robustness to extra uninformative features, and strong performance would correspond to little to no decrease in ARI metrics as the number of useless features increases. Figure \ref{fig:simulation1-bicluster-ari} summarizes the results, showing that performance of several methods deteriorates in this moderate noise regime as $p_{\text{extra}}$ increases. Notably, COBRA has mediocre performance for all values of $p_{\text{extra}}$ while BCBC falls off more slowly. The adaptive versions continue to remain robust to uninformative features.  The strongest competing methods, SparseBC and SC-Biclust, when given the true number of biclusters, have similar performance to tuned Adaptive BCBC. However, when these methods are tuned to find the number of biclusters, according to their recommended methods, Adaptive BCBC performs better, as seen in Figure \ref{fig:simulation1-bicluster-ari}. Overall we see the leading performances in true bicluster ARI is shared between Adaptive or ANN BCBC, and this gap grows larger as $p_{\text{extra}}$ increases. 

When examining the overall ability to identify useful features, we see similarities and differences across the BCBC methods. Figure \ref{fig:simulation1-auc} shows a clear decline in AUC as the balance between informative and uninformative features grows more skewed. The adaptive and ANN versions of BCBC perform similarly, despite computational efficiency gains of the approximate version. Standard BCBC has a much higher variance in its overall performance; this is likely due to the initial $k$-nearest neighbor calculation defining the affinities becoming less meaningful as $p_{\text{extra}}$ increases the number of noisy features. We see that adapting affinities while fitting leads to better performance overall. As a baseline for feature selection, we can consider pre-filtering only the top $200$ features as is recommended in COBRA \citep{chiConvexBiclustering2017}. However, even when using the true number of informative features for preprocessing, performance is worse than that under simultaneous feature selection performed by adaptive BCBC, and even the non-adaptive version of BCBC has a comparable median performance in very high noise regimes. Both SC-Biclust and SparseBC select true features at a similar rate to the adaptive BCBC methods at a low $p_{\text{extra}}$, but they begin to falter noticeably as $p_{\text{extra}}$ increases.

\begin{figure}[!ht]
    \centering
    \includegraphics[width=0.8\linewidth]{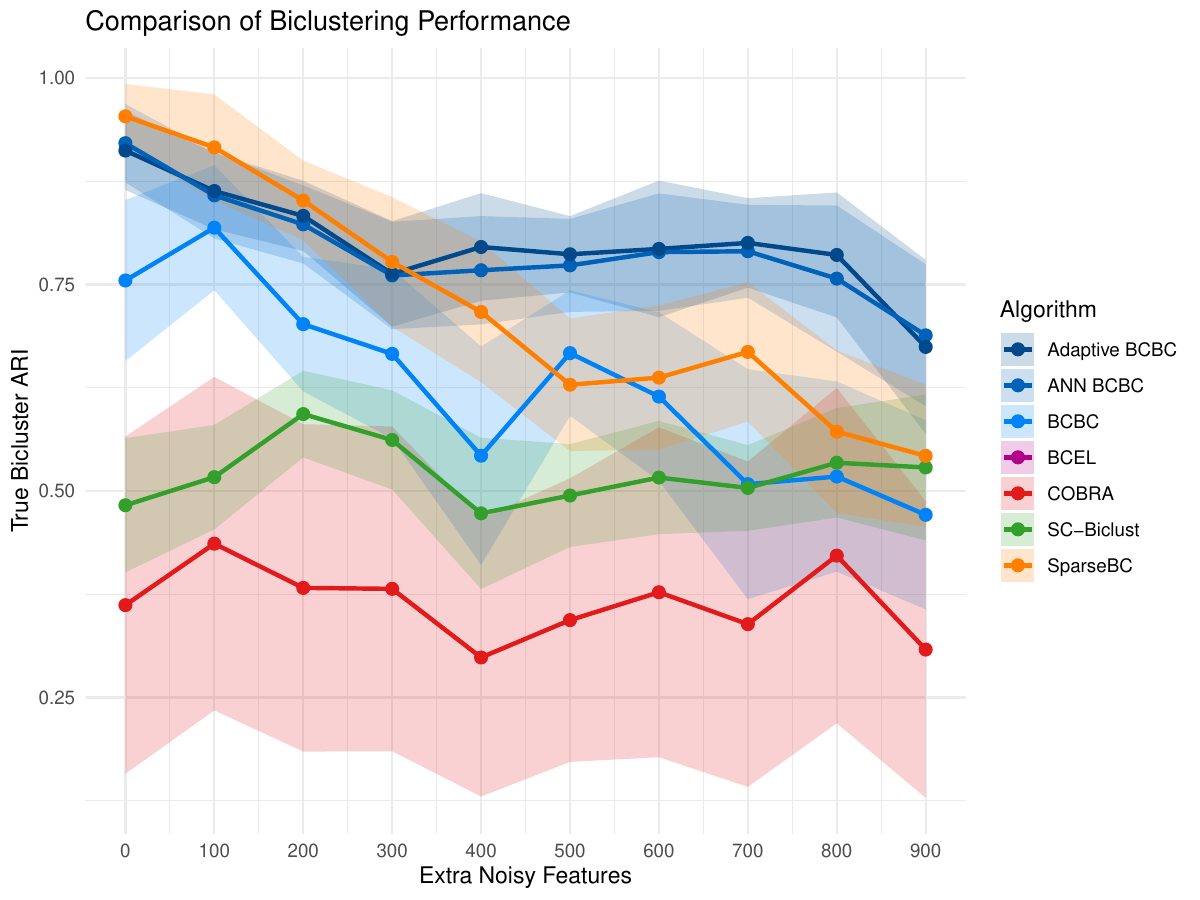}
    
    \caption{Mean of true bicluster ARI for each algorithm under the scenario described in Section \ref{section:simulation1}. The standard error is used to calculate 95\% confidence bands for each point.}
    \label{fig:simulation1-bicluster-ari}
\end{figure}

\begin{figure}[!ht]
    \centering
    \includegraphics[width=0.90\linewidth]{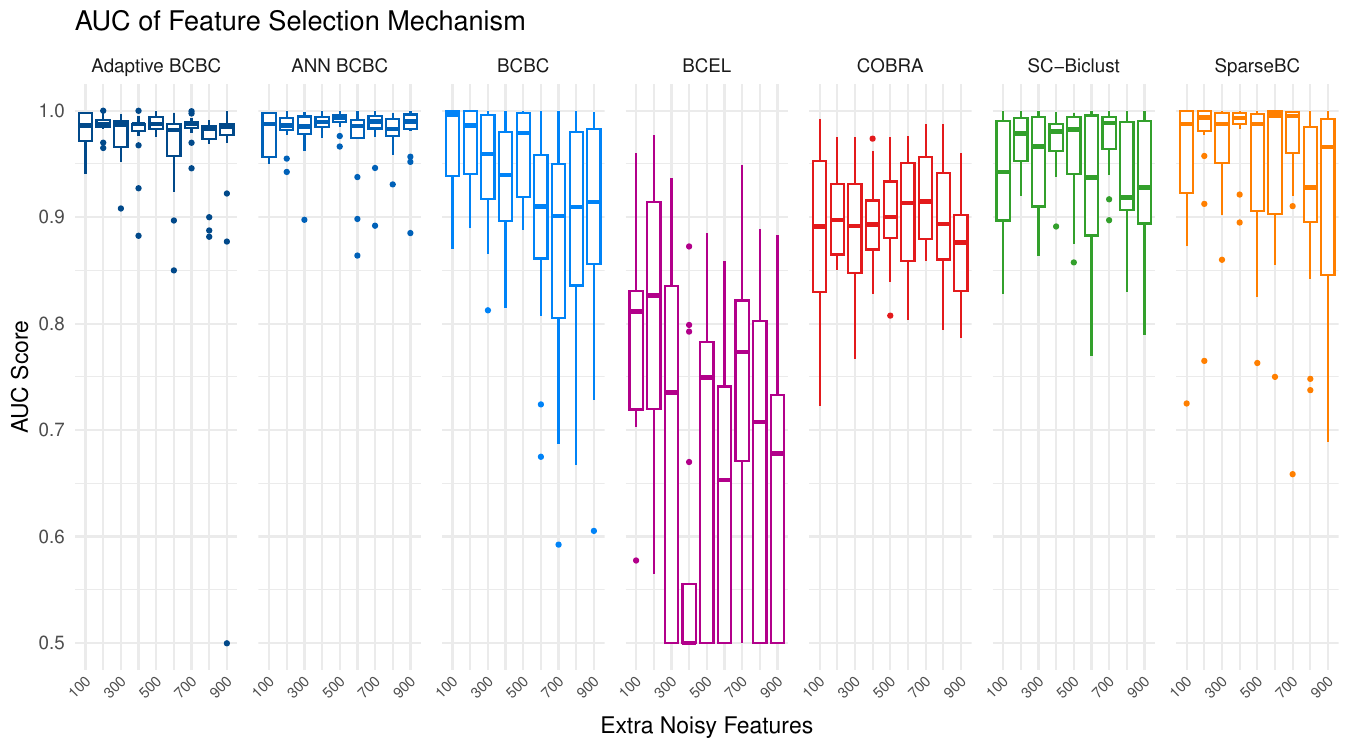}
    \caption{Distribution of AUC scores when fit under the scenario described in Section \ref{section:simulation1}.}
    \label{fig:simulation1-auc}
\end{figure}

\subsection{Varying Noise Magnitude and Dimension}\label{section:simulation2}

For our second simulation study we perform 16 trials for each value of $\sigma \in \{5, 7.5, 10, 12.5, 15\}$ and $(n_i, p_i) = (50 + 25 i, 50 + 25 i)$ for $i = 0,\ldots, 4$. Each trial generates $\bX \in \mathbb R^{n_i \times (p_i + 300)}$ which is then fit for biclusters. Every trial has an extra 300 uninformative features, and the variance of the noise is strictly greater than the variance of the bicluster centers when $\sigma > 10/\sqrt{3} \approx 5.77$. This is a very high noise regime, and we expect performance to fall off considerably.

Here, we focus our comparison of BCBC against the best performing peers, SparseBC and SC-Biclust. We use their respective tuning methods to select up to 7 possible row and column clusters, and set $k=5, \tau = 1$ for both the row and column GKNN graphs in BCBC. Figure \ref{fig:simulation2-bicluster-ari} shows that adaptive BCBC methods perform best in the majority of cases as they add robustness to otherwise poorly initialized affinities. Figure \ref{fig:simulation2-auc} shows the feature selection performance of adaptive BCBC methods is similar or better than the competing methods of SparseBC and SC-Biclust. This gap in performance is higher in the more difficult cases with a small number of informative rows and columns.

\begin{figure}[!ht]
    \centering
    \includegraphics[width=0.9\linewidth]{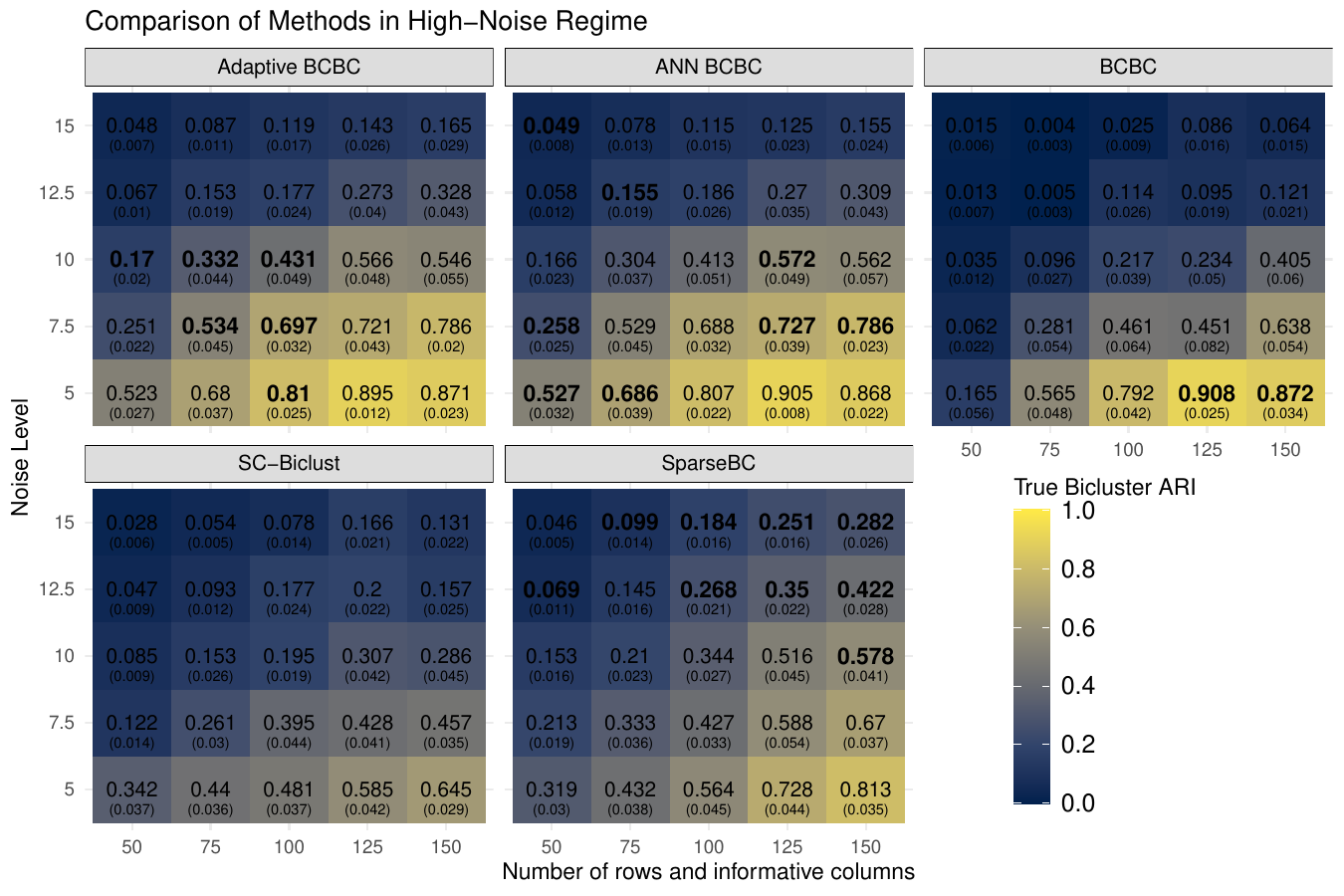}
    \caption{Results for simulations described in Section \ref{section:simulation2}: Mean true bicluster ARI. Metrics in bold are the best performance across all algorithms for the given noise variance and matrix size. The standard error is displayed below each mean in parenthesis.}
    \label{fig:simulation2-bicluster-ari}
\end{figure}

\begin{figure}[!ht]
    \centering
    \includegraphics[width=0.9\linewidth]{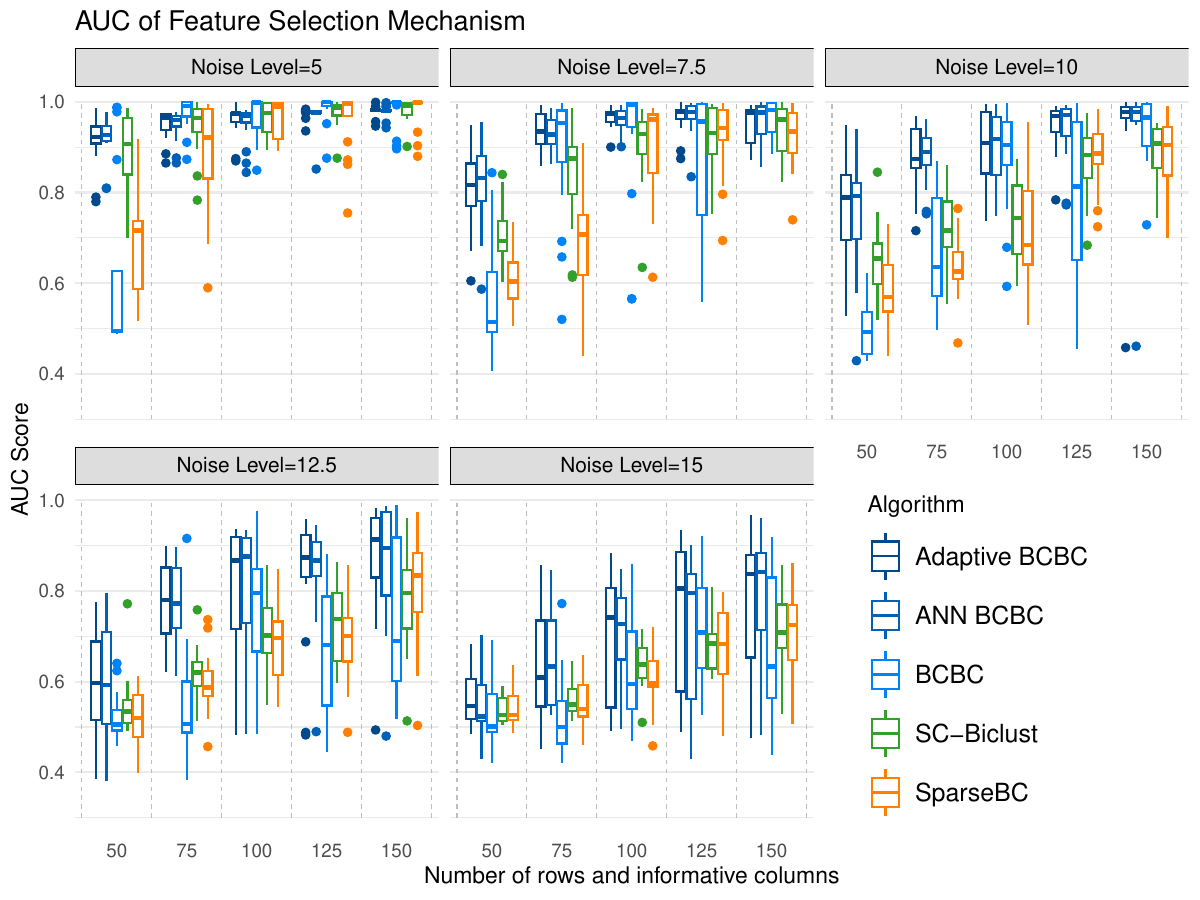}
    \caption{Distribution of AUC scores for models under the scenario described in Section \ref{section:simulation2}.}
    \label{fig:simulation2-auc}
\end{figure}

\subsection{Genomics Case Study} \label{section:real_data}

A gene microarray dataset of lymphoma samples from \citet{alizadehDistinctTypesDiffuse2000} is used as a case study with our method. In particular, we use the preprocessed array data of the log-ratios of hybridization of fluorescent probes from cDNA clones of mRNA samples compared to a reference mRNA sample. We restrict our analysis to the cDNA clones from samples of diffuse large B-cell lymphoma (DLBCL), follicular lymphoma (FL) and chronic lymphocytic leukaemia (CLL). Altogether this gives a data matrix with 62 rows (cDNA clones) and 4026 columns (mRNA samples). About 5 percent of the entries are missing due to invalid measurements.

To discover biclusters, we apply the adaptive version of BCBC with a GKNN graph using 10 and 25 nearest neighbors for the rows and columns, respectively, with $\tau = 1$. To impute the missing data, we optimize the hold-out problem \eqref{eqn:missing_objective} with $\gamma = 750$ and $\lambda = 0$. After imputation, we tune for $\lambda$ using the methods in Section \ref{section:cv} and a fixed $\gamma = 750$. Following convergence and tuning, we see a fit of five row clusters and 43 column clusters with dense weights. Some of the possible fits from tuning had sparser weights, but their eBIC was marginally larger and hence were not chosen.

Several insights can be drawn from the fit, shown in Figure \ref{fig:lymphoma}. First, many biclusters restricted to cDNA clones corresponding to CLL have opposite sign of many corresponding biclusters for DBLCL. The respective biclusters for FL seem to be in the middle, with a smaller magnitude of fluorescent probe ratios. In this case study, the learned weights from the BCBC fit are dense, so rather than selecting among possible features, they allow us to interpret them based on their importance in discriminating row clusters. Table \ref{tab:biclusters} highlights properties of the top eight column clusters by mean feature weight. 

We examine the label of the highest weight mRNA sample in each of the top eight column clusters, giving eight columns which are sufficiently different, yet can classify the row clusters in a similar manner. We find that many of the mRNA samples identified in Table \ref{tab:biclusters} are common in the oncology literature. For example, Ki-67 is a cornerstone in diagnosis of cancer \citep{liKi67PromisingMolecular2015}, Cathepsin B is associated with metastatis on lymphoma \citep{czyzewskaExpressionMatrixMetalloproteinase2008} and Beta-actin is used as a measurement reference for a wide range of cancers \citep{guoACTBCancer2013}. Particularly interesting is the mRNA sample p55CDC (member of cluster A) which has a known association with cell death in the biological literature \citep{linAnalysisCellDeath1998}. Despite being one of the mRNA samples with the most missing values, p55CDC is the second largest mRNA sample by weight in the fit, showing that our method can successfully learn its role in identifying lymphoma despite its difficulty to measure. Examining the fit further could show mRNA samples relevant to lymphoma that are not known in the literature.

\begin{longtable}[]{@{}llllp{7cm}@{}}
    \caption{Top eight column clusters from fit according to mean weight. The weights in the table are normalized such that a weight of 1 is equal to the inverse of $p=4026$. mRNA sample names are found in the dataset of \citet{alizadehDistinctTypesDiffuse2000}.}
    \label{tab:biclusters}\tabularnewline
    \toprule\noalign{}
    Label & Size & Mean Weight & Max Weight & Highest Weight mRNA Sample \\
    \midrule\noalign{}
    \endfirsthead
    \toprule\noalign{}
    Label & Size & Mean Weight & Max Weight & Highest Weight mRNA Sample \\
    \midrule\noalign{} 
    \endhead
    \bottomrule\noalign{}
    \endlastfoot
    A & 276 & 1.48 & 2.82 &  Ki-67 \\
    B & 156 & 1.26 & 2.11 & NM23-H1 (NDP kinase A) \\
    C & 319 & 1.23 & 2.05 & Cathepsin B \\
    D &  81 & 1.23 & 2.04 & Beta-actin \\
    E & 111 & 1.20 & 1.81 & Ras GTPase-activating protein \\
    F & 437 & 1.17 & 1.98 & MEK Kinase 1 (MEKK1) \\
    G &  54 & 1.14 & 1.72 & Aryl acetyltransferase \\
    H &  71 & 1.10 & 1.71 & Thioredoxin reductase
    \end{longtable}

\begin{figure}[!ht]
    \centering
    \includegraphics[width=0.9\linewidth]{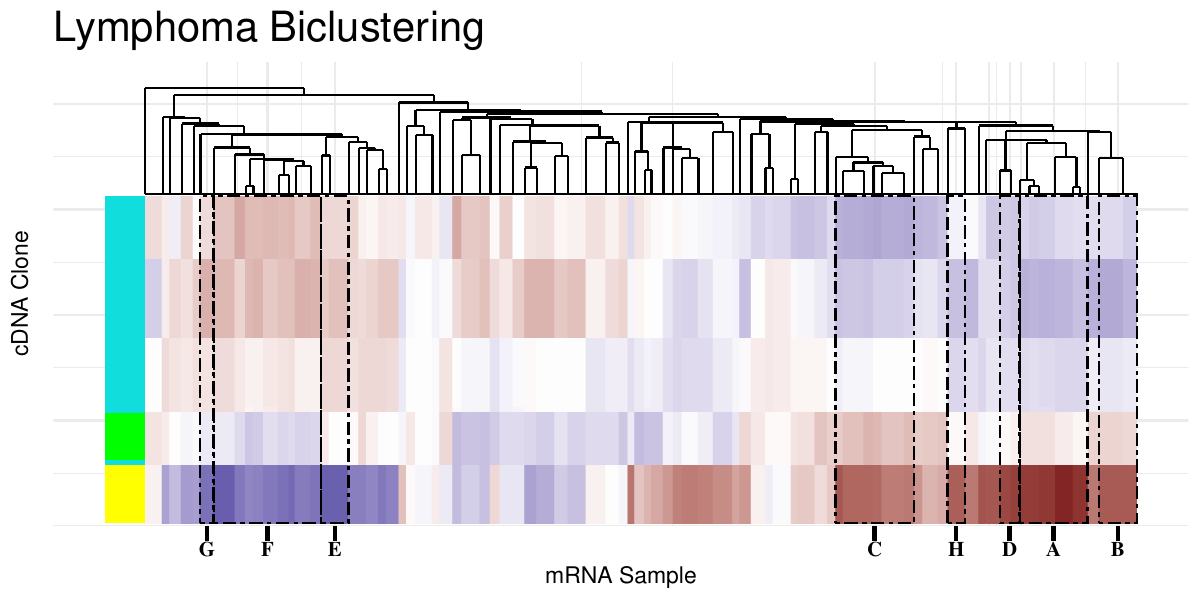}
    \caption{Biclustering results of Lymphoma data with top eight column clusters by mean weight outlined. The cDNA clone true labels correspond to DLBCL, FL, and CLL for colors cyan, green, and yellow, respectively. Blocks are labeled to be matched with Table \ref{tab:biclusters}.}
    \label{fig:lymphoma}
\end{figure}

\section{Discussion}
Through careful theoretical and empirical analyses, we find that a biconvex formulation addresses persistent challenges to biclustering in high-dimensional settings. We find that by augmenting the convex biclustering problem with a feature-weighting mechanism that is learned jointly with the clustering tasks, the method benefits from some of the stability properties of convex formulations, while lending a newly principled approach to feature weighing and selection as opposed to previously use heuristics. Indeed, despite biconvexity formally falling into the class of non-convex problems, the method nonetheless enjoys  favorable convergence guarantees as well as finite-sample guarantees for any local minimizer of the objective. These findings shed new light on the role of connectivity in the underlying affinities, which more closely aligns with what practitioners have observed in practice, and our experiments demonstrate that adaptive affinities are particularly valuable when the number of uninformative features is large. The theoretical devices may be extended to yield a better understanding of how affinity updates influence convergence behavior in future work. Because we have demonstrated how our framework successfully extends past work on biconvex formulations of clustering to the two-way task of biclustering, it is also natural to consider generalizations to tensor-valued data.

\bibliographystyle{agsm}
\bibliography{bibliography.bib}

\include{supplement}

\end{document}

%% file: supplement.tex
\appendix

\section{BCBC for Missing Data}\label{appendix:hold-out}
Recall our main objective can be written as
\begin{equation}
    \begin{split}
    F(\bU, \bw) & = \gamma \brk3{\sum_{i < j}^n \sqrt{\Phi_{ij}}\|\bU_{i\cdot}-\bU_{j\cdot}\|_2 + \sum_{k < \ell}^p \sqrt{\Psi_{k \ell} }\|\bU_{\cdot k}-\bU_{\cdot \ell}\|_2 } + \frac{1}{2} \|\bX - \bU\|^2_{\bw^2 + \lambda \bw}, \label{eqn:appendix:objective} \\
    \text{subject to } & \ \sum_{\ell = 1}^p w_\ell = 1, \quad w_\ell \geq 0.
    \end{split}
\end{equation}
Consider an observed matrix $\bX$ such that if $(i, j) \in \mathcal I$, then $X_{ij}$ is observed, otherwise it is missing. Let $\mathcal P_{\mathcal I}(\bX)$ be the projection of a matrix $\bX$ onto the indices $\mathcal I$, i.e. 
\begin{equation*}
    \mathcal P_{\mathcal I}(\bX)_{ij} = \begin{cases}
    X_{ij} & (i, j) \in \mathcal I, \\ 0 & \text{otherwise}.
\end{cases}
\end{equation*}
An analogous missing data objective to \eqref{eqn:appendix:objective} is
\begin{equation*}
    \begin{split}
        \tilde F_{\gamma, \lambda}(\bU, \bw) & := \gamma \brk3{\sum_{i < j}^n \sqrt{\Phi_{ij}}\|\bU_{i\cdot}-\bU_{j\cdot}\|_2 + \sum_{k < \ell}^p \sqrt{\Psi_{k \ell}} \|\bU_{\cdot k}-\bU_{\cdot \ell}\|_2 } \\ & \quad + \frac{1}{2}\sum_{\ell=1}^p (w_\ell^2 + \lambda w_\ell) \sum_{i = 1}^n \mathcal P_{\mathcal I}(\bU - \bX)_{i\ell}^2, \\
        \text{subject to } & \ \sum_{\ell = 1}^p w_\ell = 1, \quad w_\ell \geq 0.
    \end{split}
\end{equation*}

For brevity we write
\begin{displaymath}
    \tilde F_{\gamma, \lambda}(\bU, \bw) := \gamma f(\bU) + g(\bw) + \frac{1}{2} \| \mathcal P_{\mathcal I}(\bX) - \mathcal P_{\mathcal I}(\bU)\|^2_{\bw^2 + \lambda \bw},
\end{displaymath}
where $f$ is the fusion terms and $g$ is the constraint to the simplex for $\bw$. We can use majorization-minimization to optimize the above similar to the results in \citet{chiConvexBiclustering2017},
\begin{align}
    \tilde F_{\gamma, \lambda}(\bU, \bw) & = \gamma f(\bU) + g(\bw) + \frac{1}{2} \| \mathcal P_{\mathcal I}(\bX) - \mathcal P_{\mathcal I}(\bU)\|^2_{\bw^2 + \lambda \bw} \nonumber \\
    & \leq \gamma f(\bU) + g(\bw) + \frac{1}{2} \| \mathcal P_{\mathcal I}(\bX) - \mathcal P_{\mathcal I}(\bU)\|^2_{\bw^2 + \lambda \bw} + \frac{1}{2}\sum_{(i, j) \in \mathcal I^c} (w_j^2 + \lambda w_j)(\tilde U_{ij} - U_{ij})^2 \nonumber \\
    & = \gamma f(\bU) + g(\bw) + \frac{1}{2} \sum_{i, j} (w_j^2 + \lambda w_j)(M_{ij} - U_{ij})^2, \nonumber \\
    & =: G(\bU, \bw \mid \tilde \bU), \label{eqn:mm_obj}
\end{align}
where $\bM = \mathcal P_{\mathcal I}(\bX) + \mathcal P_{\mathcal I^c} (\tilde \bU)$. 
We can optimize $G$ by running the BCBC routine which will give estimates for the missing data entries (see Algorithm \ref{alg:missing}).

\begin{algorithm}[!h]
\caption{BCBC for Missing Data} 
\label{alg:missing}
\hspace*{0.00in} {\bf Input:} Data points $\bX$, $\mathcal I$ set of available data indices
\begin{algorithmic}[1]
    \State Initialize $k = 0$, $\bU^0_{\mathcal I} = \bX_{\mathcal I}, \bU^0_{\mathcal I^c} = \operatorname{mean}(\bX_{\mathcal I})$, $\bw^0 = \mathbf 1_p/p$.
    \While{not converged}
        \State $\bU^{k+1}, \bw^{k+1} \leftarrow \argmin_{\bU, \bw} G(\bU, \bw \mid \bU^k)$ from \eqref{eqn:mm_obj} by applying BCBC.
        \State $k \leftarrow k + 1$.
    \EndWhile
    \State \Return $\bU^k, \bw^k$.
\end{algorithmic}
\end{algorithm}

\section{\KL Property of the Objective} \label{appendix:kl}

Theorem 1 of \citet{bolteProximalAlternatingLinearized2014} shows limiting points of PALM are critical points provided the optimized function satisfies the \KL (KL) property. We now reference several definitions and facts from \citet{bolteProximalAlternatingLinearized2014} and \citet{attouchProximalAlternatingMinimization2010} describing this property.

\begin{definition}[Desingularizing Function]
    Let $\Psi_\eta$ be the set of all functions $\psi\colon [0, \eta) \mapsto \mathbb R_+$ where
    \begin{enumerate}
        \tightlist
        \item $\psi$ is concave and continuous,
        \item $\psi(0) = 0$,
        \item $\psi$ is $C^1$ on $(0, \eta)$ and continuous at 0,
        \item $\psi'(s) > 0$ for all $s \in (0, \eta).$
    \end{enumerate}
    Elements in $\Psi_\eta$ are called desingularizing functions.
\end{definition}

\begin{definition}[KL property and KL function]
    Let $f\colon \mathbb R^d \mapsto (-\infty, \infty]$ be proper and lower semicontinuous. $f$ has the \KL property at $\bar \bu \in \dom \partial f := \{\bu \in \mathbb R^d \mid \partial f(\bu) \neq \emptyset \} $ if there exist $\eta \in (0, \infty], \epsilon > 0$ and $\psi \in \Psi_\eta$, such that
    \begin{displaymath}
        \psi'[f(\bu) - f(\bar \bu)] \geq \inf \brk[c]{ \|\bv\|_2 \colon \bv \in \partial f(\bu)},
    \end{displaymath}
    whenever $\|\bu - \bar \bu\|_2 \leq \epsilon$ and $f(\bar \bu) < f(\bu) < f(\bar \bu) + \eta$. If $f$ satisfies this property at every point in $\dom \partial f$, then $f$ is called a KL function.
\end{definition}

\begin{definition}[Semialgebraic sets and functions]
    A set is semialgebraic if it can be written as a finite union of sets of the form 
    \begin{displaymath}
        \brk[c]{ \bx \in \mathbb R^d \mid p_i(\bx) = 0, q_i(\bx) < 0, i = 1, \ldots, n},
    \end{displaymath}
    where $p_i$ and $q_i$ are real polynomial functions. A function $f$ is semialgebraic if its graph $\brk[c]{ (\bx, y) \mid f(\bx) = y}$ is semialgebraic.
\end{definition}

We use two useful properties of semialgebraic functions immediately:

\begin{enumerate}
    \tightlist
    \item Finite sums of semialgebraic functions are semialgebraic.
    \item Indicator functions of semialgebraic sets are semialgebraic.
\end{enumerate}

\begin{proposition}
    $F(\bU, \bw)$ is semialgebraic and therefore a KL function.
\end{proposition}

\begin{proof}
    Recall that our objective \eqref{eqn:appendix:objective} can be written as
    \begin{displaymath}
        F(\bU, \bw) = f(\bU) + g(\bw) + H(\bU, \bw),
    \end{displaymath}
    where $f$ is the fusion terms, $g$ is the indicator for the feasibility of the weights, and $H$ is the weighted loss.
    
    Example 4 of \citet{bolteProximalAlternatingLinearized2014} shows that the $\|\cdot\|_q$ norm is semialgebraic when $q$ is rational. As a result, $f(\bU)$ is semialgebraic via the first property above.    $g(\bw)$ is semialgebraic via the second property. For completeness we show the set $\mathcal W = \brk[c]{\bw \in \mathbb R^p \mid \sum_{\ell=1}^p w_\ell = 1, w_\ell \geq 0}$ is semialgebraic via construction.
    Let $[p] = \brk[c]{1,\ldots, p}$ and $2^{[p]}$ be the resulting power set. We then have the finite union
    \begin{displaymath}
        \mathcal W = \bigcup_{\mathcal I \in 2^{[p]} \setminus \{\emptyset\}} \brk[c]3{\bw \mid \sum_{\ell \in \mathcal I} w_\ell = 1, w_\ell > 0\ \forall \ell \in \mathcal I, w_{\ell'} = 0\ \forall \ell' \in [p] \setminus \mathcal I}.
    \end{displaymath}

    Finally, consider the weighted loss, $H(\bU, \bw)$. Expanding all terms we have
    \begin{displaymath}
        H(\bU, \bw) = \sum_{\ell = 1}^p \sum_{i = 1}^n (w_\ell^2 + \lambda w_\ell) (U_{i\ell} - X_{i\ell})^2.
    \end{displaymath}
    $H$ is a polynomial function with a finite number of terms, so it is semialgebraic by definition. As a result of $f, g$, and $H$ being semialgebraic, the complete objective function $F$ is too. Thus, via Theorem 3.1 of \citet{bolteLojasiewiczInequalityNonsmooth2007a}, $F$ is a KL function and satisfies the necessary assumptions for PALM to converge to a critical point.
\end{proof}

\section{Results for Connected Affinity Graphs}\label{appendix:affinities}

\begin{lemma}[Affinity Graph Topology] \label{lemma:affinity-graph-structure}
    Let $L$ be equal to the number of nonzero entries in the upper triangle of $\bPhi \in \mathbb R^{n\times n}$. If the weighted graph, $G$, formed by the adjacency matrix $\bPhi$ is connected, then there exists a matrix $\bD$ such that

    \begin{enumerate}[label=(\roman*)]
        \tightlist
        \item $\bD \in \mathbb R^{{Lp} \times np} $, \label{lemma:fgs1}
        \item $\Phi_{ij} > 0, i < j \implies \exists \mathcal R(i, j) \text{ such that } \bD_{\mathcal R(i, j)} \vect(\bU) = \sqrt{\Phi_{ij}}(\bU_{i\cdot} - \bU_{j\cdot}) $, \label{lemma:fgs2}
        \item $\rank(\bD) = p(n-1)$, \label{lemma:fgs3}
        \item The minimum nonzero singular value of $\bD$ is $\sqrt{a(\bPhi)}$, where $a(\bPhi)$ is the algebraic connectivity of $G$, \label{lemma:fgs4}
        \item If $\bD = \bA \bLambda \bV^\top$ is the singular value decomposition of $\bD$ then the maximum singular value of $(\bA \bLambda)^\dagger$ is $\frac{1}{\sqrt{a(\bPhi)}}$ and $(\bA \bLambda)^\dagger (\bA \bLambda) = \bI_{p(n-1)}$, \label{lemma:fgs5}
        \item Let $\bV = [\bV_\beta, \bV_\alpha]$ be the right singular vectors of $\bD$, where $\bV_\beta$ and $\bV_\alpha$ correspond to the orthonormal basis of the column space and null space of $\bD$, respectively. Then $\bV_\alpha \bV_\alpha^\top = \bI_p \otimes \frac{\mathbf 1_n \mathbf 1_n^\top}{n}.$ \label{lemma:fgs6}
    \end{enumerate}
\end{lemma}

\begin{proof}
    Because $G$ is connected, $L\geq n-1$, since a graph of $n$ vertices must have at least $n-1$ edges to be connected. Furthermore, if $G$ becomes a directed graph in any orientation, then $G$ is weakly connected by definition. Let $\bB \in \{-1, 0, 1\}^{n \times L}$ be an unweighted oriented incidence matrix where any directed edge between nodes $(i, j)$ with $i < j$ starts from node $i$ and ends at node $j$. Via Theorem 9.6 of \citet{deoGraphTheoryApplications1974}, $\bB$ has rank $n-1$ since $G$ is connected. Let $\bD_1 = [\bB E(\bPhi)]^\top$ where $E(\bPhi) \in \mathbb R^{L \times L}$ is a diagonal matrix that multiplies each column of $\bB$ by the corresponding (nonzero) square root of the edge weight. Because $E(\bPhi)$ has full rank and $\bB$ is rank $(n-1)$, $\bD_1$ is also rank $n-1$. To be more explicit, if $\Phi_{ij} > 0$ and $i < j$ then there exists a \textit{row} in $\bD_1$ where $\sqrt{\Phi_{ij}}$ is in the $i$th element and $-\sqrt{\Phi_{ij}}$ is in the $j$th element. As a result, $\sqrt{\Phi_{ij}}( \bU_{i\cdot} - \bU_{j\cdot} )$ is a column in $\bD_1 \bU$.

    Let $\bD = \bI_p \otimes \bD_1$. Then we have via the Kronecker matrix-vector product, $\bD \vect(\bU) = \vect(\bD_1 \bU)$ which implies there exists some indices $\mathcal R(i,j)$ where $\bD_{\mathcal R(i, j)} \vect(\bU) = \sqrt{\Phi_{ij}}(\bU_{i\cdot} - \bU_{j\cdot})$. By construction $\bD$ satisfies \ref{lemma:fgs1} and \ref{lemma:fgs2}. Furthermore, $\rank(\bD) = \rank(\bD_1)\rank(\bI_p) = p(n-1)$ satisfying \ref{lemma:fgs3}. In addition, $\bD$ has the singular values of $\bD_1$, albeit with a higher multiplicity. We can write the Laplacian of the graph $G$ as $\bQ = \bD_1^\top \bD_1$. $Q_{ij} = -\Phi_{ij}$ if $i$ and $j$ share an edge of weight $\Phi_{ij}$ and $Q_{ii} = \sum_{j} \Phi_{ij}$. For \ref{lemma:fgs4}, the smallest non-zero eigenvalue of this generalized Laplacian is known as the algebraic connectivity, which is equal to the square of the smallest singular value of $\bD_1$ and hence $\bD$ \citep{deabreuOldNewResults2007}.

    Let $\bD = \bA \bLambda \bV^\top$ be the singular value decomposition of $\bD$ such that $\bA \in \mathbb R^{Lp \times p(n-1)}, \bLambda \in \mathbb R^{p(n-1)\times p(n-1)}$ and $\bV \in \mathbb R^{np \times p(n-1)}$. Let $\bZ = \bA \bLambda$. By definition, $\bZ$ has the same singular values as $\bD$. Furthermore, since $L \geq n-1$, $\bA^\top \bA = \bI_{p(n-1)}$ so there exists $\bZ^\dagger$ such that $\bZ^\dagger \bZ = \bI_{p(n-1)}$. In addition, $\Lambda_{\max}(\bZ^\dagger) = \frac{1}{\Lambda_{\min}(\bZ)}$ completing the proof of \ref{lemma:fgs5}.

    Because the sums of all rows of $\bD_1$ are zero, $\bD_1\mathbf 1_n = \bzero_{L}$. Since $G$ is connected giving $\bD_1$ a rank of $n-1$, the null space of $\bD_1$ has dimension $1$ with $\mathbf 1_n / \sqrt{n}$ as the only basis unit vector. This implies $\frac{\mathbf 1_n \mathbf 1_n^\top}{n}$ is a projection matrix onto the null space of $\bD_1$. Let $\bD_1 = \bA_1 \bLambda_1 \bV_1^\top$ be the SVD giving $\bD = (\bI_p \otimes \bA_1)(\bI_p \otimes \bLambda_1)(\bI_p \otimes \bV_1^\top)$. Therefore, $\bV_\alpha \bV_\alpha^\top = \bI_p \otimes \frac{\mathbf 1_n \mathbf 1_n^\top}{n}$ is a projection matrix onto the null space of $\bD$, proving \ref{lemma:fgs6}. 
\end{proof}

\begin{lemma}[Deterministic Affinity Graph Error Bounds]\label{lemma:connectivity_lemma}
    Let $\vect(\bX) = \vect(\bU) + \bepsilon$, where $\bepsilon$ are independent sub-Gaussian random variables with variance $\sigma^2$. Let $\bD$ be such that $\sum_{i < j} \|\bD_{\mathcal R(i, j)} \vect(\bU) \|_2 = \sum_{i < j} \sqrt{\Phi_{ij}} \|\bU_{i\cdot} - \bU_{j \cdot}\|_2$, where $\mathcal R$ is some index set and $\bD = \bA \bLambda \bV_\beta^\top$ is the singular value decomposition. If the graph formed by the deterministic adjacency matrix $\bPhi$ is connected then
    \begin{displaymath}
        P\brk[s]3{\|\bepsilon^\top \bV_\beta (\bA \bLambda)^\dagger \|_\infty \geq 2\sigma \sqrt{\frac{\log(p |\mathcal E_{\bPhi}|)}{a(\bPhi)}} } \leq \frac{2}{p |\mathcal E_{\bPhi }|},
    \end{displaymath}
    where $|\mathcal E_{\bPhi}|$ is the number of edges in the graph formed by $\bPhi$.
\end{lemma}

\begin{proof} Similar to the proof of Lemma 6 in \citet{tanStatisticalPropertiesConvex2015} we use basic properties of sub-Gaussian random variables and a union bound. Let $\be_j$ be a unit vector of length $p|\mathcal E_{\bPhi}|$ with a 1 in the $j$th element. Let $v_j = \bepsilon^\top \bV_\beta (\bA \bLambda)^\dagger \be_j$. Clearly $\|\bepsilon^\top \bV_\beta (\bA \bLambda)^\dagger \|_\infty = \max_j |v_j|$. We have $\bLambda_{\max}(\bV_\beta) = 1$ and $\bLambda_{\max} [(\bA \bLambda)^\dagger] = \frac{1}{\sqrt{a(\bPhi)}}$ by Lemma \ref{lemma:affinity-graph-structure}. As a consequence, $v_j$ is sub-Gaussian with zero mean and variance at most $\frac{\sigma^2}{a(\bPhi)}$. By a union bound,
\begin{displaymath}
    P\brk3{\max_j |v_j| \geq z} \leq p |\mathcal E_{\bPhi}| P\brk{|v_j| \geq z} \leq 2 p |\mathcal E_{\bPhi}| \exp\brk[s]3{-\frac{z^2 a(\bPhi)}{2\sigma^2}}.
\end{displaymath}
Choose $z = 2\sigma \sqrt{\frac{\log (p |\mathcal E_{\bPhi}|)}{a(\bPhi)}}$, to complete the proof.
\end{proof}

\section{Proof of Theorem 1}\label{appendix:main_proofs}

The Hanson-Wright inequality is common in the convex clustering literature and we restate it here for completeness.

\begin{lemma}[Hanson-Wright Inequality \citep{hansonBoundTailProbabilities1971}] \label{lemma:hanson-wright}
    Let $\bepsilon$ be a sub-Gaussian random vector with independent components of mean 0 and variance $\sigma^2$. If $\bA$ is a symmetric matrix, then there exists a constant $C > 0$ such that for any $t \geq 0$,
    \begin{displaymath}
        P\brk[s]{\bepsilon^\top \bA \bepsilon \geq t + \sigma^2 \trace(\bA)} \leq \exp\brk[s]3{-C\min\brk3{\frac{t^2}{\sigma^4\|\bA\|^2_F}, \frac{t}{\sigma^2 \|\bA\|_2}}}.
    \end{displaymath}
\end{lemma}

We proceed with the proof.

\begin{proof}

Let $L_R := \|\bPhi\|_0/2 = |\mathcal E^R|$ and $L_C := \|\bPsi\|_0 / 2 = |\mathcal E^C|$ be the number of edges from the row and column affinity graphs, respectively. Let $\bx := \vect(\bX), \bu := \vect(\bU)$ be the vectorizations of $\bX, \bU$, respectively.
Let $\bD^R \in \mathbb R^{[L_Rp] \times np}$ be such that $\bD^R_{\mathcal R(i, j)}\bu = \sqrt{\Phi_{ij}}(\bU_{i\cdot} - \bU_{j\cdot})$ if $\Phi_{ij}> 0$ where $\mathcal R(i, j)$ is an index set. Let $\mathcal E^R = \{(i, j) \mid \Phi_{ij} > 0, i < j\}$ be the set of all edges for the row affinities. Via Lemma \ref{lemma:affinity-graph-structure}, we know $\bD^R$ exists and has rank $p(n-1)$. Let $\bD^R := \bA^R \bLambda^R (\bV^R_{\beta})^\top$ be the singular value decomposition of $\bD^R$. Explicitly, $\bA^R \in \mathbb R^{L_Rp \times p(n-1)}, \bLambda^R \in \mathbb R^{p(n-1)\times p(n-1)}$ and $\bV_\beta^R \in \mathbb R^{np \times p(n-1)}$. Via properties of the SVD we can write $\bV^R := [\bV^R_\alpha, \bV^R_\beta] \in \mathbb R^{np \times np}$ where $\bV^R_\alpha$ and $\bV^R_\beta$ are orthonormal matrices where $(\bV^R_\alpha)^\top \bV^R_\beta = \bzero$. Let the following symbols be defined:
\begin{align*}
    \bbeta^R & := (\bV_\beta^R)^\top \bu, \\
    \balpha^R & := (\bV_\alpha^R)^\top \bu, \\
    \bZ^R & := \bA^R \bLambda^R.
\end{align*}
These symbols imply
\begin{align*}
    \bu & = \bV_\beta^R \bbeta^R + \bV_\alpha^R \balpha^R, \\
    \bD^R \bu & = \bA^R \bLambda^R (\bV_\beta^R)^\top \bu = \bZ^R \bbeta^R.
\end{align*}

Similarly for the columns, let $\bD^C \in \mathbb R^{[L_Cn]\times np}$ be such that $\bD^C_{\mathcal C(k, \ell)}\bu = \sqrt{\Psi_{k\ell}}( \bU_{\cdot k} - \bU_{\cdot \ell})$, if $\Psi_{k\ell} > 0$ where $\mathcal C(k, \ell)$ is an index set. Via Lemma \ref{lemma:affinity-graph-structure}, $\bD^C$ has rank $n(p-1)$, as all the same conditions hold for $\bPsi$, and we are examining the transpose of $\bU$. Let $\mathcal E^C = \{(k, \ell) \mid \Psi_{k\ell} > 0, k < \ell \}$ be the set of all edges for the column affinities. Let the following analogous symbols be defined:
\begin{align*}
    \bD^C & = \bA^C \bLambda^C (\bV_\beta^C)^\top, \\
    \bA^C & \in \mathbb R^{L_Cn \times n(p-1)}, \bLambda^C \in \mathbb R^{n(p-1)\times n(p-1)}, \bV_\beta^C \in \mathbb R^{np \times n(p-1)}, \\
    \bV^C & := [\bV^C_\alpha, \bV^C_\beta], \quad (\bV_\alpha^C)^\top \bV^C_\beta = \bzero, \\
    \bbeta^C & := (\bV_\beta^C)^\top \bu, \\
    \balpha^C & := (\bV_\alpha^C)^\top \bu, \\
    \bZ^C & := \bA^C \bLambda^C, \\
    \bu & = \bV_\beta^C\bbeta^C + \bV_\alpha^C \balpha^C, \\
    \bD^C \bu & = \bA^C \bLambda^C (\bV_\beta^C)^\top \bu = \bZ^C \bbeta^C.
\end{align*}

Let
\begin{displaymath}
    \mathcal W = \brk[c]3{\bW \in \mathbb R^{np\times np} \mid \bW = \diag(w_1, \ldots, w_p) \otimes \bI_n, w_1,\ldots,w_p \geq 0, \sum_{\ell = 1}^p w_\ell = 1},
\end{displaymath}
be the set of matrices allowing $\|\bx - \bu\|^2_{\bW^2 + \lambda \bW} = (\bx - \bu)^\top (\bW^2 + \lambda \bW) (\bx - \bu)$ to be equal to the weighted norm used in \eqref{eqn:appendix:objective}. Let $\gamma' = \frac{\gamma}{np}$. Then objective \eqref{eqn:appendix:objective} is equivalent to
\begin{equation}
    \begin{split}
        G(\balpha^R, \balpha^C, \bbeta^R, \bbeta^C, \bW) & = \frac{1}{4np} \|\bx - \bV^R_\alpha \balpha^R - \bV^R_\beta \bbeta^R\|^2_{\bW^2 + \lambda \bW} + \gamma' \sum_{(i, j) \in \mathcal E^R} \|\bZ^R_{\mathcal R(i, j)} \bbeta^R\|_2 \\
        & \quad + \frac{1}{4np} \|\bx - \bV^C_\alpha \balpha^C - \bV^C_\beta \bbeta^C \|^2_{\bW^2 + \lambda \bW}  + \gamma' \sum_{(k, \ell) \in \mathcal E^C} \|\bZ^C_{\mathcal C(k, \ell)} \bbeta^C\|_2, \\
        \text{subject to } & \bW \in \mathcal W,\quad \bV^R [\balpha^R, \bbeta^R] = \bV^C[\balpha^C, \bbeta^C].  
    \end{split} \label{eqn:equiv_obj}
\end{equation}

Let $\hat \bU$ and $\hat \bw$ be local minimizers of \eqref{eqn:appendix:objective}. Due to convexity in the $\bU$ block when $\hat \bw$ is fixed (recall $F$ is biconvex), $\hat \bU$ is the minimizer with $\hat \bw$ as the input:
\begin{equation*}
    \begin{split}
    & \frac{1}{np} H(\hat \bU, \hat \bw) + \gamma' \sum_{(i, j) \in \mathcal E^R} \sqrt{\Phi_{ij}} \|\hat \bU_{i\cdot} - \hat \bU_{j\cdot}\|_2 + \gamma' \sum_{(k, \ell) \in \mathcal E^C} \sqrt{\Psi_{k\ell}} \|\hat \bU_{\cdot k} - \hat \bU_{ \cdot \ell}\|_2 \\
    & \leq \frac{1}{np} H(\bU, \hat \bw) + \gamma' \sum_{(i, j) \in \mathcal E^R} \sqrt{\Phi_{ij}}\| \bU_{i\cdot} - \bU_{j\cdot}\|_2 + \gamma' \sum_{(k, \ell) \in \mathcal E^C}\sqrt{\Psi_{k\ell}} \|    \bU_{\cdot k} - \bU_{ \cdot \ell}\|_2.
    \end{split}
\end{equation*}

Recall $\bX = \bU + \bE$:
\begin{equation*}
    \begin{split}
    & \frac{1}{2np} \sum_{\ell=1}^p \brk{\hat w_{\ell}^2 + \lambda \hat w_\ell}\|\bU_{\cdot \ell} + \bE_{\cdot \ell} - \hat \bU_{\cdot \ell}\|_2^2 \\
    & \quad + \gamma' \sum_{(i, j) \in \mathcal E^R} \sqrt{\Phi_{ij}} \|\hat \bU_{i\cdot} - \hat \bU_{j\cdot}\|_2 + \gamma' \sum_{(k, \ell) \in \mathcal E^C} \sqrt{\Psi_{k\ell}} \|\hat \bU_{\cdot k} - \hat \bU_{ \cdot \ell}\|_2 \\
    & \leq \frac{1}{2np} \sum_{\ell=1}^p {\brk{\hat w_{\ell}^2+ \lambda \hat w_\ell }\|\bE_{\cdot \ell}\|_2^2} \\
    & \quad + \gamma' \sum_{(i, j) \in \mathcal E^R} \sqrt{\Phi_{ij}} \| \bU_{i\cdot} - \bU_{j\cdot}\|_2 + \gamma' \sum_{(k, \ell) \in \mathcal E^C} \sqrt{\Psi_{k\ell}} \|\bU_{\cdot k} - \bU_{ \cdot \ell}\|_2.
    \end{split}
\end{equation*}

Now expand all norms and collect like terms:
\begin{equation}
    \begin{split}
    & \frac{1}{2np} \sum_{\ell=1}^p \brk{\hat w_\ell^2 + \lambda \hat w_\ell}\|\bU_{\cdot \ell}- \hat \bU_{\cdot \ell} \|_2^2 \\
    & \quad + \gamma' \sum_{(i, j) \in \mathcal E^R} \sqrt{\Phi_{ij}} \|\hat \bU_{i\cdot} - \hat \bU_{j\cdot}\|_2 + \gamma' \sum_{(k, \ell) \in \mathcal E^C} \sqrt{\Psi_{k\ell}} \|\hat \bU_{\cdot k} - \hat \bU_{ \cdot \ell}\|_2 \\
    & \leq \frac{1}{np} \sum_{\ell=1}^p \brk{\hat w_{\ell}^2 + \lambda \hat w_\ell }\bE_{\cdot \ell}^\top (\hat \bU_{\cdot \ell} - \bU_{\cdot \ell}) \\
    & \quad + \gamma' \sum_{(i, j) \in \mathcal E^R} \sqrt{\Phi_{ij}} \| \bU_{i\cdot} - \bU_{j\cdot}\|_2 + \gamma' \sum_{(k, \ell) \in \mathcal E^C} \sqrt{\Psi_{k\ell}} \|\bU_{\cdot k} - \bU_{ \cdot \ell}\|_2.
    \end{split} \label{eqn:inequality_to_bound}
\end{equation}

Let $\hat \bW = \hat \bw \otimes \bI_n \in \mathcal W$. Let $\hat \balpha^R, \hat \bbeta^R, \hat \balpha^C, \hat \bbeta^C$ be $\argmin G(\cdot, \cdot, \cdot, \cdot, \hat \bW)$ from $\eqref{eqn:equiv_obj}$ so $\hat \bu = \bV_\alpha^R \hat \balpha^R + \bV_\beta^R \hat \bbeta^R = \bV_\alpha^C \hat \balpha^C + \bV_\beta^C \hat \bbeta^C$. Let $\bepsilon = \vect(\bE)$. We can rewrite \eqref{eqn:inequality_to_bound} as:
\begin{equation}
    \begin{split}
    & \sum_{M \in \{R, C\}} \frac{1}{4np} \|\bV^M_\alpha ( \balpha^M - \hat \balpha^M) + \bV^M_\beta ( \bbeta^M - \hat \bbeta^M) \|^2_{\hat \bW^2 + \lambda \hat \bW} + \gamma' \sum_{(i, j) \in \mathcal E^M} \|\bZ^M_{\mathcal M(i, j)} \hat \bbeta^M\|_2 \\
    & \leq \sum_{M \in \{R, C\}} \frac{1}{2 np} \bepsilon^\top (\hat\bW^2 + \lambda \hat \bW) [\bV^M_\alpha ( \hat \balpha^M -  \balpha^M) + \bV^M_\beta ( \hat \bbeta^M - \bbeta^M)] + \gamma' \sum_{(i, j) \in \mathcal E^M} \|\bZ^M_{\mathcal M(i, j)} \bbeta^M\|_2. \\
    \end{split} \label{eqn:inequality_rewritten}
\end{equation}
$\hat \balpha^M$ with $M \in \{R, C\}$, is a minimizer of \eqref{eqn:equiv_obj} when the $\bW$ block is $\hat \bW$. To be a coordinate-wise minimizer, the gradient with respect to $\balpha^M$ of \eqref{eqn:equiv_obj} must equal 0:
\begin{align*}
    0 & = \nabla_{\balpha^M} \brk3{ \frac{1}{4np} \|\bx - \bV^M_\alpha \balpha^M - \bV^M_\beta \bbeta^M \|^2_{\hat \bW^2 + \lambda \hat \bW}}, \\
    & = \frac{1}{2np} (\bV^M_\alpha)^\top (\hat \bW^2 + \lambda \hat \bW) ( \bx - \bV^M_\alpha \balpha^M - \bV^M_\beta \bbeta^M).
\end{align*}
This implies a $\hat \balpha^M$ such that $ \bx - \hat \bV^M_\alpha \hat \balpha^M - \hat \bV^M_\beta \hat \bbeta^M = \bzero$. By using that $\bx = \bu + \bepsilon$, $\bu = \bV_\alpha^M \balpha^M + \bV_\beta^M \bbeta^M$, $(\bV_\alpha^M)^\top \bV_\beta^M = \bzero$ and $ (\bV_\alpha^M)^\top \bV_\alpha = \bI$, this gives $\hat \balpha^M = (\bV^M_\alpha)^\top (\bx - \bV^M_\beta \hat \bbeta^M) = \balpha^M + (\bV^M_\alpha)^\top \bepsilon$ and
\begin{equation}
    \bV^M_\alpha ( \hat \balpha^M -  \balpha^M) = \bV^M_\alpha (\bV^M_\alpha)^\top \bepsilon. \label{eqn:alpha-diff-sub}
\end{equation}

We simplify the RHS of \eqref{eqn:inequality_rewritten} with \eqref{eqn:alpha-diff-sub} as:
\begin{equation*}
    \begin{split}
    & \sum_{M \in \{R, C\}} \frac{1}{2 np} \bepsilon^\top (\hat\bW^2 + \lambda \hat \bW) [\bV^M_\alpha ( \hat \balpha^M -  \balpha^M) + \bV^M_\beta ( \hat \bbeta^M - \bbeta^M)] + \gamma' \sum_{(i, j) \in \mathcal E^M} \|\bZ^M_{\mathcal M(i, j)} \bbeta^M\|_2, \\
    & = \sum_{M \in \{R, C\}} \frac{1}{2np} \underbrace{\bepsilon^\top (\hat \bW^2 + \lambda \hat \bW) \bV^M_\alpha (\bV^M_\alpha)^\top \bepsilon}_{(I)} + \frac{1}{2np}\underbrace{\bepsilon^\top (\hat \bW^2 + \lambda \hat \bW) \bV^M_\beta ( \hat \bbeta^M - \bbeta^M)}_{(II)} \\
    & \quad + \gamma' \sum_{(i, j) \in \mathcal E^M} \|\bZ^M_{\mathcal M(i, j)} \bbeta^M\|_2. 
    \end{split}
\end{equation*}
We bound term ($I$) in the summation using
\begin{align*}
    \bepsilon^\top (\hat \bW^2 + \lambda \hat \bW) \bV^M_\alpha (\bV^M_\alpha)^\top \bepsilon & = \trace\brk[s]2{\bepsilon^\top (\hat \bW^2 + \lambda \hat \bW) \bV^M_\alpha (\bV^M_\alpha)^\top \bepsilon}, \\
    & = \trace\brk[s]2{(\hat \bW^2 + \lambda \hat \bW) \bV^M_\alpha (\bV^M_\alpha)^\top \bepsilon \bepsilon^\top }, \\
    & = \sum_{i=1}^{np} (\hat W_{ii}^2 + \lambda \hat W_{ii}) \brk[s]{\bV^M_\alpha (\bV^M_\alpha)^\top \bepsilon \bepsilon^\top}_{ii}, \\
    & \leq \|(\hat \bW^2 + \lambda \hat \bW)\|_\infty \trace \brk[s]2{\bV^M_\alpha (\bV^M_\alpha)^\top \bepsilon \bepsilon^\top }, \\
    & \leq (1 + \lambda) { \bepsilon^\top \bV^M_\alpha (\bV^M_\alpha)^\top \bepsilon }.
\end{align*}

To bound ($II$) we use Lemma \ref{lemma:affinity-graph-structure} for both the row and column modes. As a result, there exists a pseudoinverse $(\bZ^M)^\dagger$ where $(\bZ^M)^\dagger\bZ^M = \bI_{\rank(\bZ^M)}$. Insert this expression into ($II$): 
\begin{align*}
    \bepsilon^\top (\hat \bW^2 + \lambda \hat \bW) \bV^M_\beta ( \hat \bbeta^M - \bbeta^M) & =     \bepsilon^\top (\hat \bW^2 + \lambda \hat \bW) \bV^M_\beta (\bZ^M)^\dagger \bZ^M (\hat \bbeta^M - \bbeta^M), \\
    & = \sum_{(i, j) \in \mathcal E^M} \bepsilon^\top (\hat \bW^2 + \lambda \hat \bW) \bV^M_\beta (\bZ^M_{\mathcal M(i, j)})^\dagger \bZ^M_{\mathcal M(i, j)} (\hat \bbeta^M - \bbeta^M), \\
    & \leq \sum_{(i, j) \in \mathcal E^M} \norm2{\bepsilon^\top (\hat \bW^2 + \lambda \hat \bW) \bV^M_\beta (\bZ^M_{\mathcal M(i, j)})^\dagger}_2 \\
    & \quad \times \norm2{\bZ^M_{\mathcal M(i, j)} (\hat \bbeta^M - \bbeta^M)}_2, \\
    & \leq \max_{(i, j) \in \mathcal E^M} \norm2{\bepsilon^\top (\hat \bW^2 + \lambda \hat \bW) \bV^M_\beta (\bZ^M_{\mathcal M(i, j)})^\dagger}_2  \\
    & \quad \times \sum_{(i, j) \in \mathcal E^M} \norm2{\bZ^M_{\mathcal M(i, j)} (\hat \bbeta^M - \bbeta^M)}_2.
\end{align*}
We can bound $\max_{(i, j) \in \mathcal E^M} \norm2{\bepsilon^\top (\hat \bW^2 + \lambda \hat \bW) \bV^M_\beta (\bZ^M_{\mathcal M(i, j)})^\dagger}_2$ by using the equivalence of the 2-norm and the $\infty$-norm. For simplicity, consider the mode of $M=R$
\begin{align*}
    \max_{(i, j) \in \mathcal E^R} \norm2{\bepsilon^\top (\hat \bW^2 + \lambda \hat \bW) \bV^R_\beta (\bZ^R_{\mathcal R(i, j)})^\dagger}_2 & \leq \sqrt{p} \max_{(i, j) \in \mathcal E^R} \norm2{\bepsilon^\top (\hat \bW^2 + \lambda \hat \bW) \bV^R_\beta (\bZ^R_{\mathcal R(i, j)})^\dagger}_\infty, \\
    & = \sqrt{p} \|\bepsilon^\top (\hat \bW^2 + \lambda \hat \bW) \bV^R_\beta (\bZ^R)^\dagger \|_\infty.
\end{align*}
Now let $\boldsymbol e_j$ be the vector of length $p L_R$ with a one in the $j$th entry and zero elsewhere. Then
\begin{displaymath}
    \|\bepsilon^\top (\hat \bW^2 + \lambda \hat \bW) \bV^R_\beta (\bZ^R)^\dagger \|_\infty = \max_j |\bepsilon^\top (\hat \bW^2 + \lambda \hat \bW) \bV^R_\beta (\bZ^R)^\dagger \boldsymbol e_j|.
\end{displaymath}
Bounding in a similar manner to ($I$):
\begin{align*}
    \max_j |\bepsilon^\top (\hat \bW^2 + \lambda \hat \bW) \bV^R_\beta (\bZ^R)^\dagger \boldsymbol e_j| & = \max_j \trace \brk[s]2{(\hat \bW^2 + \lambda \hat \bW) \bV^R_\beta (\bZ^R)^\dagger \boldsymbol e_j \bepsilon^\top}, \\
    & = \max_j \sum_{i=1}^{np} (\hat W_{ii}^2 + \lambda \hat W_{ii}) \brk[s]2{\bV^R_\beta (\bZ^R)^\dagger \boldsymbol e_j \bepsilon^\top}_{ii}, \\
    & \leq  \|(\hat \bW^2 + \lambda \hat \bW)\|_\infty \max_j \trace\brk[s]2{ \bV^R_\beta (\bZ^R)^\dagger \boldsymbol e_j \bepsilon^\top}, \\
    & \leq  (1 + \lambda) \max_j \bepsilon^\top \bV^R_\beta (\bZ^R)^\dagger \boldsymbol e_j.
\end{align*}

To summarize our results so far by using all bounds on \eqref{eqn:inequality_to_bound},
\begin{equation}
    \begin{split}
        & \frac{1}{2np}\| \bU - \hat \bU\|_{\hat \bw^2 + \lambda \hat \bw}^2 + \gamma' \sum_{(i, j) \in \mathcal E^R} \sqrt{\Phi_{ij}} \|\hat \bU_{i\cdot} - \hat \bU_{j\cdot}\|_2 + \gamma' \sum_{(k, \ell) \in \mathcal E^C} \sqrt{\Psi_{k\ell}} \|\hat \bU_{\cdot k} - \hat \bU_{ \cdot \ell}\|_2 \\
        & \leq \frac{1 + \lambda}{2np} \sum_{M \in \{R, C\}} \max_j [ \bepsilon^\top \bV^M_\beta (\bZ^M)^\dagger \boldsymbol e_j ] \sum_{(i, j) \in \mathcal E^M} \norm2{\bZ^M_{\mathcal M(i, j)} (\hat \bbeta^M - \bbeta^M)}_2 \begin{cases} \sqrt{p} & M = R \\ \sqrt{n} & M = C\end{cases} \\
        & \quad + {\bepsilon^\top \bV^M_\alpha (\bV^M_\alpha)^\top \bepsilon} + \gamma' \sum_{(i, j) \in \mathcal E^M} \|\bZ^M_{\mathcal M(i, j)} \bbeta^M\|_2.
    \end{split} \label{eqn:last-a.s.}
\end{equation}
Note that this inequality holds almost surely, and we now begin considering inequalities that hold with high probability.

\textbf{Bounding $\frac{1 + \lambda}{2np} {\bepsilon^\top \bV^M_\alpha (\bV^M_\alpha)^\top \bepsilon}$ with high probability}:

Due to properties of projection matrices $\|\bV^R_\alpha (\bV_\alpha^R)^\top\|_F^2 = p, \|\bV^C_\alpha (\bV_\alpha^C)^\top\|_F^2 = n$, and $\|\bV^R_\alpha (\bV_\alpha^R)^\top\|_2 = \|\bV^C_\alpha (\bV_\alpha^C)^\top\|_2 = 1 $ since both matrices are of rank $p$ and $n$, respectively. Now apply the standard Hanson-Wright inequality (Lemma \ref{lemma:hanson-wright}) for both the row and column modes with $t = \sigma^2 \sqrt{p \log(np)}$ and $t = \sigma^2 \sqrt{n \log(np)}$, respectively:
\begin{align*}
    P\brk[s]3{\frac{1 + \lambda}{2np}\bepsilon^\top \bV^R_\alpha (\bV_\alpha^R)^\top \bepsilon \geq \frac{\sigma^2(1 + \lambda)}{2np}\brk2{\sqrt{p \log(np)} + p} } & \leq  \exp\brk[c]3{-C \min \brk[s]2{\log(np),  \sqrt{p \log(np)}}}, \\
    P\brk[s]3{\frac{1 + \lambda}{2np}\bepsilon^\top \bV^C_\alpha (\bV_\alpha^C)^\top \bepsilon \geq \frac{\sigma^2(1 + \lambda)}{2np}\brk2{\sqrt{n \log(np)} + n} } & \leq  \exp\brk[c]3{-C \min \brk[s]2{\log(np),  \sqrt{n \log(np)}}}.
\end{align*}

\textbf{Bounding $\frac{1 + \lambda}{2np} \max_j \bepsilon^\top \bV^M_\beta (\bZ^M)^\dagger \boldsymbol e_j \begin{cases} \sqrt{p} & M = R \\ \sqrt{n} & M = C\end{cases}$ with high probability:}

For simplicity, let $M = R$. Via Lemma \ref{lemma:connectivity_lemma} we have 
\begin{displaymath}
    P\brk[s]3{\max_j |\bepsilon^\top \bV^R_\beta (\bZ^R)^\dagger \boldsymbol e_j| \geq  2\sigma \sqrt{\frac{\log(p L_R)}{a(\bPhi)}}} \leq \frac{2}{p L_R}.
\end{displaymath}
This implies if ${\gamma'} > \sigma(1+\lambda) \sqrt{\frac{\log(p L_R)}{n^2p a(\bPhi)}}$, then ${\gamma'} < \frac{1+\lambda}{2n\sqrt{p}}\max_j |\bepsilon^\top \bV^R_\beta (\bZ^R)^\dagger \boldsymbol e_j|$ with probability at most $\frac{2}{pL_R}$. Similarly for the column mode if ${\gamma'} > \sigma(1+\lambda) \sqrt{\frac{\log(n L_C)}{n p^2 a(\bPsi)}}$, then ${\gamma'} < \frac{1+\lambda}{2p\sqrt{n}}\max_j |\bepsilon^\top \bV^C_\beta (\bZ^C)^\dagger \boldsymbol e_j|$ with probability at most $\frac{2}{nL_C}$. Combining all these bounds, we now bound the RHS of \eqref{eqn:last-a.s.}:
\begin{equation*}
    \begin{split}
  & \frac{1 + \lambda}{2np} \sum_{M \in \{R, C\}} \max_j \bepsilon^\top \bV^M_\beta (\bZ^M)^\dagger \boldsymbol e_j \sum_{(i, j) \in \mathcal E^M} \norm2{\bZ^M_{\mathcal M(i, j)} (\hat \bbeta^M - \bbeta^M)}_2 \begin{cases} \sqrt{p} & M = R \\ \sqrt{n} & M = C\end{cases} \\
        & \quad + {\bepsilon^\top \bV^M_\alpha (\bV^M_\alpha)^\top \bepsilon} + \gamma' \sum_{(i, j) \in \mathcal E^M} \|\bZ^M_{\mathcal M(i, j)} \bbeta^M\|_2 \\
     & \leq \frac{\sigma^2 (1+\lambda)}{2} \brk[s]3{\frac{1}{n} + \frac{1}{p} + \sqrt{\frac{\log(np)}{n p^2}} + \sqrt{\frac{\log(np)}{n^2 p}}} + \gamma' \sum_{(i, j) \in \mathcal E^R} \norm2{\bZ^R_{\mathcal R(i, j)} (\hat \bbeta^R - \bbeta^R)}_2 \\
        & \quad + \gamma' \sum_{(k, \ell) \in \mathcal E^C} \norm2{\bZ^C_{\mathcal C(k, \ell)} (\hat \bbeta^C - \bbeta^C)}_2 + \gamma' \sum_{(i, j) \in \mathcal E^M} \|\bZ^M_{\mathcal M(i, j)} \bbeta^M\|_2
    \end{split}
\end{equation*}
with probability at least 
\begin{displaymath}
    1 - \exp\brk[c]3{-C \min \brk[s]2{\log(np),  \sqrt{p \log(np)}}} - \exp\brk[c]3{-C \min \brk[s]2{\log(np),  \sqrt{n \log(np)}}} - \frac{2}{pL_R} - \frac{2}{n L_C}
\end{displaymath} if $\gamma' > \frac{\sigma(1+\lambda)}{\sqrt{np}} \max \brk3{\sqrt{\frac{\log(p L_R)}{n a(\bPhi)}}, \sqrt{\frac{\log(n L_C)}{{p a(\bPsi)}}}} $. This probability can be further simplified since $L_R \geq n-1$ and $L_C \geq p - 1$:
\begin{align*}
    \frac{2}{p L_R} & \leq \frac{2}{p(n-1)} \leq \frac{4}{np}, \\
    \frac{2}{n L_C} & \leq \frac{2}{n(p-1)} \leq \frac{4}{np}.
\end{align*}
Using these simpler bounds and substituting into \eqref{eqn:last-a.s.}, with probability at least $$1 - 2 \exp\brk[c]3{-C \min \brk[s]2{\log(np),  \sqrt{\min(n, p) \log(np)}}} - \frac{8}{np}$$ and large enough $\gamma'$:
\begin{equation*}
    \begin{split}
        & \frac{1}{2np} \|\bU - \hat \bU\|_{\hat \bw^2 + \lambda \hat \bw}^2 + \gamma' \sum_{(i, j) \in \mathcal E^R} \|\bZ^R_{\mathcal R(i, j)} \hat \bbeta^R\|_2 + \gamma' \sum_{(k, \ell) \in \mathcal E^C} \|\bZ^C_{\mathcal C(k, \ell)} \hat \bbeta^C\|_2 \\
        & \leq \frac{\sigma^2 (1+\lambda)}{2} \brk[s]3{\frac{1}{n} + \frac{1}{p} + \sqrt{\frac{\log(np)}{n p^2}} + \sqrt{\frac{\log(np)}{n^2 p}}} + \gamma' \sum_{(i, j) \in \mathcal E^R} \norm2{\bZ^R_{\mathcal R(i, j)} (\hat \bbeta^R - \bbeta^R)}_2 \\
        & \quad + \gamma' \sum_{(k, \ell) \in \mathcal E^C} \norm2{\bZ^C_{\mathcal C(k, \ell)} (\hat \bbeta^C - \bbeta^C)}_2 + \gamma' \sum_{(i, j) \in \mathcal E^M} \|\bZ^M_{\mathcal M(i, j)} \bbeta^M\|_2.
    \end{split}
\end{equation*}
Now subtract the fusion terms from the LHS and use triangle inequality to obtain the desired bound:
\begin{equation*}
    \begin{split}
        \frac{1}{2np} \|\bU - \hat \bU\|_{\hat \bw^2 + \lambda \hat \bw}^2 & \leq \frac{\sigma^2 (1+\lambda)}{2} \brk[s]3{\frac{1}{n} + \frac{1}{p} + \sqrt{\frac{\log(np)}{n p^2}} + \sqrt{\frac{\log(np)}{n^2 p}}} \\
        & \quad + 2\gamma' \sum_{i < j} \sqrt{\Phi_{ij}} \|\bU_{i \cdot} - \bU_{j \cdot}\|_2 \\
        & \quad + 2\gamma' \sum_{k < \ell} \sqrt{\Psi_{k\ell}} \|\bU_{\cdot k} - \bU_{\cdot \ell}\|_2.
    \end{split}
\end{equation*}

\end{proof}

\section{Cluster assignments}\label{appendix:cluster-assignments}
 We recommend assigning two rows to the same cluster if the weighted distance between them is smaller than a threshold value $r_1$, e.g.
\begin{equation*}
    \bU_{i\cdot}^\star(r_1) = \bU_{j\cdot}^\star(r_1) \iff \|\hat \bU_{i\cdot} - \hat \bU_{j \cdot}\|_{\hat \bw^2 + \lambda \hat \bw} \leq r_1,
\end{equation*}
where $\hat \bU$ and $\hat \bw$ are local minima of \eqref{eqn:appendix:objective}, and $\bU^\star$ is the least squares estimate for $\bU$ with biclusters determined by threshold radius $r_1$. This is a similar approach to \citet{chiConvexBiclustering2017}, which recommends setting $r_1$ to be a percentage of the standard deviation of all pairwise weighted row distances. This is equivalent to building a network where each node represents a row, and edges are present if two nodes fall within the calculated threshold. Subsequently, the connected components of the network will determine cluster membership. An analogous method is used to identify the column clusters, while assigning all columns of zero weight to an isolated bicluster, e.g. 
\begin{equation*}
    \bU_{\cdot k}^\star(r_2) = \bU_{\cdot \ell}^\star(r_2) \iff \|\hat \bU_{\cdot k} - \hat \bU_{\cdot \ell}\|_{2} \leq r_2 \text{ and } \hat \bw_k, \hat \bw_\ell > 0.
\end{equation*}
After calculating all row and column cluster memberships, the least square estimate with these assignments can be calculated
\begin{equation}
\begin{split}
    \bU^\star_{ik}(r_1, r_2) & = \operatorname{mean}(\{X_{j\ell} \colon \bU^\star_{i\cdot}(r_1) = \bU^\star_{j\cdot}(r_1), \bU^\star_{\cdot k}(r_2) = \bU^\star_{\cdot \ell}(r_2)\}), \quad (\bw_k > 0); \\
    \bU^\star_{ik}(r_1, r_2) & = \operatorname{mean}(\{X_{j\ell} \colon \hat \bw_\ell = 0\}), \quad (\bw_k = 0). 
\end{split}
\end{equation}
